\documentclass[twoside]{article}

\usepackage[accepted]{aistats2024}

\usepackage{url}
\usepackage{algorithm,algpseudocodex}
\usepackage{enumerate}
\usepackage{xspace}

\usepackage{setspace}
\usepackage{amsmath,amsfonts,amssymb,amsthm}
\usepackage{thmtools,thm-restate}
\usepackage{amsmath}
\usepackage[hidelinks]{hyperref}
\usepackage{mathrsfs}
\usepackage[final]{microtype}
\usepackage{bm}
\usepackage{multirow}
\usepackage{caption}
\usepackage{thmtools}
\usepackage{thm-restate}
\usepackage{mathtools}
\usepackage[noabbrev,capitalize]{cleveref}
\usepackage{graphicx,subcaption}

\usepackage{custom_citations}
\bibliography{refs}

\newcommand\numberthis{\addtocounter{equation}{1}\tag{\theequation}}

\def\##1\#{\begin{align}#1\end{align}}
\def\$#1\${\begin{align*}#1\end{align*}}

\DeclareMathOperator*{\argmax}{arg\,max}
\DeclareMathOperator*{\argmin}{arg\,min}
\DeclareMathOperator*{\diag}{diag}
\DeclareMathOperator*{\supp}{supp}
\DeclareMathOperator*{\Var}{Var}
\DeclareMathOperator*{\trace}{trace}

\newcommand{\Np}{\N^+}

\newcommand{\E}{\mathbb{E}}
\newcommand{\EE}[1]{\mathbb{E}[#1]}
\renewcommand{\P}{\mathbb{P}}
\newcommand{\spaced}[1]{\quad \text{#1} \quad}

\newcommand{\cX}{\mathcal{X}}
\newcommand{\cE}{\mathcal{E}}
\newcommand{\cF}{\mathcal{F}}
\newcommand{\cG}{\mathcal{G}}

\newcommand{\N}{\mathbb{N}}
\newcommand{\R}{\mathbb{R}}

\newcommand{\cU}{\mathcal{U}}
\newcommand{\cL}{\mathcal{L}}
\newcommand{\cT}{\mathcal{T}}

\newcommand{\cN}{\mathcal{N}}
 
\newcommand{\tran}{^\top}
\newcommand{\cQ}{\mathcal{Q}}
\newcommand{\F}{\mathbb{F}}
\newcommand{\EC}{\E_{t-1}}

\newcommand{\cD}{\mathcal{D}}

\newcommand{\1}[1]{1\{ #1\}}
\newcommand{\one}[1]{\1{#1}}

\declaretheorem[name=Theorem]{theorem}
\declaretheorem[name=Example]{example}
\declaretheorem[name=Lemma]{lemma}
\declaretheorem[name=Proposition]{proposition}

\declaretheorem[name=Claim]{claim}
\declaretheorem[name=Remark]{remark}
\declaretheorem[name=Definition]{definition}
\declaretheorem[name=Assumption]{assumption}

\usepackage[disable]{todonotes}

\let\norm\undefined % <-- "Undefine" \norm
\newcommand{\norm}[2][]{\|#2\|_{#1}} 
\newcommand{\snorm}[1]{\|#1\|} % simple norm

\newcommand{\ThetaOpt}{\Theta^{\textrm{opt}}} 
\newcommand{\tP}{\widetilde \P_{t-1}}
\newcommand{\hH}{\widehat H}

\newcommand{\rphe}{r^{\mathrm{OPT}}_t} %{r^{\text{PHE}}_t}
\newcommand{\rpred}{r^{\mathrm{PRED}}_t}
\newcommand{\wR}{\widetilde R(n)}
\newcommand{\Ht}{H_t} %{H(\theta_\star; \cD_t)}

\begin{document}

\runningauthor{David Janz, Shuai Liu, Alex Ayoub, Csaba Szepesvári}

\twocolumn[

\aistatstitle{Exploration via linearly perturbed loss minimisation}

\aistatsauthor{ David Janz\textsuperscript{\ensuremath{\star}} \And Shuai Liu\textsuperscript{\ensuremath{\star}} \And Alex Ayoub\textsuperscript{\ensuremath{\star}} \And Csaba Szepesvári\textsuperscript{\ensuremath{\star}}}

\aistatsaddress{University of Alberta \And  University of Alberta \And University of Alberta \And University of Alberta}]

\begin{abstract}
We introduce \emph{exploration via linear loss perturbations} (EVILL), a randomised exploration method for structured stochastic bandit problems that works by solving for the minimiser of a linearly perturbed regularised negative log-likelihood function. We show that, for the case of generalised linear bandits,  EVILL reduces to perturbed history exploration (PHE), a method where exploration is done by training on randomly perturbed rewards. In doing so, we provide a simple and clean explanation of when and why random reward perturbations give rise to good bandit algorithms. We propose data-dependent perturbations not present in previous PHE-type methods that allow EVILL to match the performance of Thompson-sampling-style parameter-perturbation methods, both in theory and in practice. Moreover, we show an example outside generalised linear bandits where PHE leads to inconsistent estimates, and thus linear regret, while EVILL remains performant. Like PHE, EVILL can be implemented in just a few lines of code.
\end{abstract}

\paragraph{Update} The final steps of the proof of our main result (\cref{thm:main}) relies on an erroneous argument inherited from the proof for Thompson sampling of
\citet[Appendix~D.2]{faury2022jointly}. A corrected argument for Thompson sampling is given by \citet{perneczky2026variance}; adapting that argument to our exploration via linear loss perturbations algorithm recovers the stated regret bound. All claims in the main body of this paper remain correct, but the values of the constants in the appendix should not be relied upon. See \cref{eq:rphe-first-bound} on \cpageref{eq:rphe-first-bound} and bolded text around it for the precise error. See \citet{perneczky2026variance} for further discussion.

\section{INTRODUCTION}
Effective exploration is key to the success of algorithms that learn to optimise long term reward while interacting with their environments \citep{lattimore2018bandit}.
Algorithms based on \emph{perturbed history exploration} (PHE)
follow the optimal policy given a model fitted to randomly perturbed data. The appeal of PHE is that, alike to the widely used $\epsilon$-greedy strategy,
it is \emph{minimally invasive}: PHE can be implemented by adding a few extra lines of code on top of that used to fit models and optimise policies against fitted models.

\begin{NoHyper}
	\renewcommand{\thefootnote}{\ensuremath{\star}}
	\footnotetext{\texttt{\{djanz,shuai14,aayoub,szepesva\}@ualberta.ca}}
\end{NoHyper}

The challenge in using PHE is to design perturbations that induce sufficient exploration
while controlling the computational cost. Two strategies have been proposed for this purpose. First, \citet{kveton2019perturbed,kve20:PHElinear} proposed to
add a fixed number of copies of each observation, but with the rewards replaced by some randomly chosen ones.
An alternative, which avoids increasing the size of the data,
is to \emph{additively perturb the observed rewards}
\citep{kveton2019randomized}.
\citeauthor{kveton2019randomized} motive this latter approach---a special case of
randomised least-squares value iteration \citep{osb16rlsvi}---by `the equivalence of
posterior sampling and perturbations by Gaussian noise in linear models
\citep{lu2017ensemble}, when the prior of $\theta_\star$ and rewards are Gaussian'. In this paper, we follow up on the work of \citet{kveton2019randomized} and ask when and why will additive reward perturbations induce the right amount of exploration in structured, nonlinear bandits.

The first difficulty that one encounters when attempting to answer this question is best illustrated by considering \emph{logistic bandit problems},
where the rewards associated with the individual arms are binary-valued.
Then,
when real-valued rewards are used in place of binary rewards---which is what happens when the observed binary rewards are additively perturbed by Gaussian noise---the log-likelihood associated with the data becomes \emph{undefined}.
\citet{kveton2019randomized} note this, but also observe that
the term in the loss that causes the problem is constant (as a function of the unknown parameter to be estimated), and hence can be dropped without changing the estimate.
While convenient, this critically relies on the properties of logistic bandits, namely that the reward distribution is a member of the natural exponential family of distributions. For more general problems, additive reward perturbations may lead to algorithms that are not well-defined.

Moreover, even when additive reward perturbations lead to a well-defined method,
it is unclear why these are reasonable.
While in the case of linear bandits, additive reward perturbations relate closely to posterior sampling, which has a firm theoretical basis \citep{agrawal2013thompson,abeille2017linear}, no such justification exists for, say, generalised linear bandits (GLBs). Indeed, for GLBs, the only theoretical result for PHE, that of \citet{kveton2019randomized}, relies on the troublesome assumption that the arms of the bandit are orthogonal---an assumption under which the problem becomes equivalent to that of learning and acting in an unstructured multi-armed bandit. Thus, \citet{kveton2019randomized} pose as an open problem the construction and analysis of an appropriate PHE-style data-perturbation method for structured, nonlinear bandit problems.

Our work both resolves the above two challenges, and provides a new view on additive reward perturbations. We take a step back, and design a new algorithm,
\emph{exploration via linear loss perturbations} (EVILL), which minimises loss functions that are perturbed by adding random linear components.
We prove that this perturbed loss is \emph{equivalent} to the loss used by \citet{kveton2019randomized} in \emph{generalised linear bandits}, but is applicable even when PHE with additive perturbations is not well-defined. We justify our linear perturbations by considering a quadratic approximation to the negative log-likelihood, and show how to set these in a way such that the minimiser of the perturbed loss gives rise to an optimistic model with constant probability, a property that is widely recognised as key to the
design of effective random exploration methods for stochastic environments \citep{agrawal2013thompson,abeille2017linear,lattimore2018bandit}.

In contrast with the PHE method for GLBs of \citet{kveton2019randomized}, the perturbations induced by EVILL have a data-dependent variance that scales with an estimate of the Fisher information of the reward distribution. This scaling arises naturally by considering a quadratic approximation to the loss, and similar scaling was used in previous works that approximate Thompson sampling \citep[e.g.,][]{kveton2019randomized}. Our experiments show that these data-dependent perturbations lead to a significant boost in performance, and demonstrate a specific bandit problem outside the class of GLBs where PHE with additive data perturbations leads to \emph{inconsistent} estimates and \emph{linear regret}, while the estimates of EVILL are consistent and the algorithm is performant.

To summarise, we extend and strengthen the literature on randomised exploration in stochastic bandits as follows:
\begin{itemize}
	\item We propose a new way of inducing exploration, EVILL, which adds a random linear term to the
	model-fitting loss, chosen to induce optimism.
	\item We show that EVILL is equivalent to PHE with additive perturbations in the generalised linear bandit setting, but is also applicable in settings where PHE is not well-defined or is inconsistent.
	\item We propose that the equivalence of EVILL and PHE is the best way to understand when and why additive reward perturbations can be successful.
	\item We establish that, in self-concordant generalised linear bandits,
	the regret of EVILL enjoys guarantees that parallel those available for Thompson sampling.
	Due to the equivalence of EVILL and PHE, this resolves the open problem of giving a rigorous theoretical justification for PHE with additive reward perturbations.
	\item We establish experimentally that our
	data-dependent perturbations are effective and that EVILL is competitive with alternatives.
\end{itemize}
In addition, we establish that the concentration results of \textcite{abeille2021instance,russac2021self} hold in all self-concordant generalised linear bandits, removing a previous assumption of bounded responses.

\section{PROBLEM SETTING}
\label{sec:problemsetting}
In this section we introduce notation, and describe the bandit problems we consider. These will, for now, be quite general, which facilitates the description of our algorithm and its relation to PHE\@. Our upcoming theoretical results will consider the narrower setting, that of generalised linear bandits equipped with some extra technical assumptions, described in \cref{sec:theory}.

\paragraph{Notation} Throughout, $\snorm{\cdot}$ will denote the Euclidean 2-norm. For a positive semidefinite matrix $M$ and a vector $v$ of compatible dimensions, $\norm[M]{v}$ will denote the $M$-weighted Euclidean 2-norm given by $\norm[M]{v}^2 = v\tran M v$. For any positive integer $n \in \N^+$, $[n]$ denotes the integers $\{1, \dotsc, n\}$.  We use $\R$ to denote the set of reals.
We let $\diag(a_1,\dots,a_d)$ stand for the $d\times d$ diagonal matrix with diagonal entries $a_1,\dots,a_d$. We use $I_d$ to denote the $d\times d$ identity matrix.
We use upper-case letters to denote random variables, which are defined over a probability space $(\Omega,\cF,\P)$ with expectation operator $\E$. We write $\supp(Q)$ to denote the support of a distribution $Q$. For a real-valued differentiable map $f$ on a suitable subset of $\R^d$, we use $\dot{f}$ ($\ddot{f}$, $\dddot{f}$, $\dots$) to denote the first (respectively, second, third, etc.) derivative of $f$, and view $\dot{f}$ as a column vector. We denote the interior of a set $\cU\subset \R$ by $\cU^\circ$.

\paragraph{Bandits \& regret}
We consider bandit problems described by a non-empty (possibly infinite) arm set $\cX$, a parameter vector $\theta_\star\in \Theta \subset \R^d$,
a family of smooth \emph{localisation maps} $(\phi_x)_{x\in \cX}$ with $\phi_x: \Theta  \to \R$,
and a family of distributions $(P(\cdot;u))_{u\in \cU}$ over the reals,
where $\cU\subset \R$ is a (possibly infinite) interval of real numbers with non-empty interior such that for all $x\in \cX$, $\phi_x(\theta_\star)\in \cU$. We will use $\mu(u)$ to denote the mean of $P(\,\cdot\,;u)$:
\[
\mu(u) = \int y\,\, P(dy;u)\,.
\]

We assume that the learner is given the family $(P(\,\cdot\,;u))_{u\in \cU}$, the arm set $\cX$, the family of maps $(\phi_x)_{x\in \cX}$, but not the parameter $\theta_\star$.
In each round $t \in\{1,2, \dotsc\}$ of interaction, the learner selects an action $X_t \in \cX$
based on the available past information.
In response, the learner receives a random reward $Y_t \in \R$, which they observe.
The distribution of $Y_t$, given the past, and in particular $X_t$, is
$P(\,\cdot\,; \phi_{X_t}(\theta_\star))$.
The goal of the learner is to maximise its mean reward.
The performance of the learner is measured by their (pseudo)regret,
which, for an interaction that lasts for $n$ rounds, is given
\begin{equation*}
	R(n) = \sum_{t=1}^n \mu(\phi_{x_\star}(\theta_\star)) - \mu(\phi_{X_t}(\theta_\star))\,,
\end{equation*}
where $x_\star$ is any arm in $\argmax_{x \in \cX} \mu(\phi_x(\theta_\star))$.
For simplicity, we assume the existence of the maximising arm $x_\star$.

To make the above concrete, consider how the generalised linear bandits fit into our framework:
\begin{example}[Generalised linear bandits \citep{filippi2010parametric}]
	\label{ex:glm}
	Generalised linear bandits (GLBs) are obtained by choosing $(P(\,\cdot\,;u))_{u\in \cU}$ to be a natural exponential family of distributions (for more on these, see \cref{sec:nef}),
	$\cX\subset \R^d$, $\Theta = \R^d$, and $\phi_x(\theta) = x\tran\theta$.
\end{example}

\paragraph{Regularised negative log-likelihood}

The algorithms we design for the learner will rely on regularised maximum likelihood estimation,
or, equivalently, on minimising a regularised negative log-likelihood loss.

To define the likelihood function, assume that each distribution in the family $(P(\cdot;u))_{u\in \cU}$ is dominated by a common $\sigma$-finite measure $\nu$ over the reals. This allows us to consider the density of $P(\cdot;u)$ with respect to $\nu$, which we will denote by $p(\cdot;u)$; that is,
$p(y;u)  = \frac{d P(\cdot;u)}{d\nu}$. Then, given a list of pairs $\cD = ( (x_1,y_1),\dots,(x_s,y_s))$ of length $s$ where $x_i\in \cX$ and $y_i\in \R$,
and a regularisation parameter $\lambda>0$,
the $\lambda$-regularised \emph{negative} log-likelihood function is defined to be
\begin{equation}
	\cL(\theta; \cD) = \frac{\lambda}{2} \|\theta\|^2 -\sum_{i=1}^s \log p(y_i; \phi_{x_i}(\theta))\,.
	\label{eq:nll}
\end{equation}
We assume throughout that the minimisation of $\theta\mapsto \cL(\theta;\cD)+w^\top \theta$ is efficient for  data that can be generated by the model and any arbitrary weight vector $w\in \R^d$. This holds, for example, for generalised linear bandits.

\paragraph{Examples} To wrap up our setting section, we give two more examples of suitable bandit problems.
\begin{example}[Logistic linear bandits]
\label{ex:llb}
This is a special case of generalised linear bandits when $\cU = \R$ and $P(\,\cdot\,;u)$ is the Bernoulli distribution with mean\[\mu(u) =\exp(u)/(1+\exp(u))\,.\] Then, picking $\nu$ to be the counting measure over $\{0,1\}$, we have the density \[p(y;u) = (\mu(u))^y (1-\mu(u))^{1-y}\] for $y\in \{0,1\}$.
\end{example}

\begin{example}[Linear bandits with Gaussian rewards]
\label{ex:lgb}
Another special case of generalised linear bandits is when $\cU = \R$ and
$P(\,\cdot\,;u)$ is a Gaussian distribution with (say) variance one and mean $\mu(u)=u$. Here, taking $\nu$ to be the Lebesgue measure over the reals, we have a density \[p(y;u) = \exp(-(y-u)^2/2)/\sqrt{2\pi}\,.\]
\end{example}

We will give one further example in \cref{sec:equiv}, that of a Rayleigh linear bandit, on which the previous PHE approaches lead to inconsistent estimation and linear regret. Rayleigh linear bandits fall outside GLBs.

\section{THE EVILL ALGORITHM}
\label{sec:alg}
%!TEX spellcheck = en_UK
%!TEX root =  main.tex
\begin{algorithm*}[tbh]
    \caption{\label{alg:EVILL} Exploration via linear loss perturbations (EVILL)}
    \begin{algorithmic}[1]\onehalfspacing
        \Require 
        horizon $n \in \N^+$, arm-set $\cX$, a feasible parameter set $\Theta'\subset\R^d$, 
        localisation maps $(\phi_x)_{x\in \cX}$, \\
        Fisher information map $I\colon\cU^\circ \to[0,\infty)$ and mean map $\mu \colon \cU \to \R$,
        regularisation parameter $\lambda > 0$, \\
        perturbation scale $a > 0$,
        prior observations $((X_i, Y_i))_{i=1}^\tau$ of length $\tau \in [n]$
		\ForAll{$t \in \{\tau+1, \dotsc, n\}$}
			\State  Compute $\hat\theta_{t-1} \in \argmin_{\theta \in \Theta'} \cL(\theta; ((X_i, Y_i))_{i=1}^{t-1} )$		\Comment{MLE of $\theta_\star$}\label{alg:mle}
			\State   Sample perturbations $Z_t \sim \cN(0,I_d)$, $Z_t'\sim \cN(0,I_{t-1})$\label{alg:sample}
            \State Compute vector $W_{t} = a \lambda^{1/2} Z_t +
             a \sum_{i=1}^{t-1} I(\phi_{X_i}(\hat\theta_{t-1}))^{1/2} Z_{t,i}' \dot\phi_{X_i}(\hat\theta_{t-1})$\label{alg:weight} 
			\State  Compute $\theta_t \in \argmin_{\theta \in  \Theta'}
			 \cL(\theta; ((X_i, Y_i))_{i=1}^{t-1} )+W_{t}^\top \theta$		
			 \Comment{Optimise the linearly perturbed loss}\label{alg:mlep}
			\State  Select arm $X_t \in \argmax_{x \in \cX} \mu( \phi_x(\theta_t))$ and receive reward $Y_t$\label{alg:arm}
		\EndFor
	\end{algorithmic}
\end{algorithm*}

Consider the EVILL algorithm, listed in \cref{alg:EVILL}.
EVILL takes as arguments a horizon length $n \in \N^+$, a regularisation parameter $\lambda>0$,
a perturbation scale parameter $a> 0$ and
prior observations $\cD_{\tau}=((X_i, Y_i))_{i=1}^\tau$ of length $\tau \le n$,
where $\tau$ can be a stopping time (i.e., chosen in a data-dependent way).
In addition, the algorithm needs to have access to the maps $(\phi_x)_{x\in \cX}$, the
arm-set $\cX$, a feasible region for the parameters $\Theta'\subset \R^d$,
the Fisher-information map
$I:\cU^\circ \to [0,\infty)$, and the mean maps $\mu:\cU \to [0,\infty)$ underlying the
single-parameter reward distribution family $(P(\cdot;u))_{u\in \cU}$.
We define the Fisher information map shortly, and discuss its importance in \cref{sec:why-linear}.

The role of prior observations $\cD_\tau$ is to ensure
that, right from the start, for all arms $x\in \cX$,
$\phi_x(\theta_\star)$ is estimated up to a constant accuracy by $\phi_x(\hat\theta_{\tau})$, where
$\smash{\hat\theta_\tau \in \argmin_{\theta\in \Theta'} \cL(\theta; \cD_\tau)}$. The prior observations may be offline data, already available before any interaction occurs, or may be collected using a standard warm-up procedure, as discussed in \cref{apx:warm-up}.

With the prior observations in-place, at each round  $t \in \{\tau+1, \dotsc, n\}$,
EVILL first finds $\hat\theta_{t-1}$, the minimiser of $\cL(\cdot;\cD_{t-1})$, where
$\cD_{t-1} = ((X_i,Y_i))_{i=1}^{t-1}$ collects all past data. That is, we find
\begin{equation*}
	\hat\theta_{t-1} \in \argmin_{\theta\in \Theta'}  \cL(\cdot;\cD_{t-1})\,.
\end{equation*}
This preliminary estimate $\hat\theta_{t-1}$ then is used to construct a random perturbation vector, $W_{t}$, given by
\begin{equation*}
W_{t} = a \lambda^{1/2} Z_t +
             a \sum_{i=1}^{t-1} I(\phi_{X_i}(\hat\theta_{t-1}))^{1/2} Z_{t,i}' \dot\phi_{X_i}(\hat\theta_{t-1})\,,
\end{equation*}
where
 $Z_t \sim \cN(0,I_d)$, $Z_t'\sim \cN(0,I_{t-1})$,
$\dot\phi_x$ is the derivative of $\phi_x$ and
$I:\cU^\circ \to [0,\infty)$ is the Fisher-information map underlying $(P(\cdot;u))_{u\in \cU}$, given by
\begin{equation}
I(u) = \int  \frac{\partial^2}{\partial u^2} \log \frac{1}{p(y;u)} P(dy;u)\,. \label{eq:fisher}
\end{equation}

Next, the loss minimiser is invoked on the loss
perturbed by $W_{t}^\top \theta$, to compute
\begin{align*}
\theta_t \in \argmin_{\theta\in \Theta'} \cL(\theta;\cD_{t-1}) + W_{t}^\top \theta\,,
\end{align*}
and $\theta_t$ is
used to select the action $X_t$ as
\begin{equation}
	X_t \in \argmax_{x \in \cX} \mu( \phi_x(\theta_t) )\,. \label{eq:bestarm1}
\end{equation}
Again, merely to simplify the exposition, we assume that the minimisers and this latter maximiser exist.

If $\mu$ is increasing, which holds in natural exponential families,
the choice in \cref{eq:bestarm1} is the same as
$X_t \in \argmax_{x \in \cX} \phi_x(\theta_t)$.
When $\phi_x$ is the bilinear map as in generalised linear models, this simplifies to
$X_t \in \argmax_{x \in \cX} x^\top \theta_t$,
while $\dot\phi_{X_i}(\hat\theta_{t-1})$ (used in the construction of $W_{t}$) simplifies to $X_i$.

In generalised linear bandits, under mild regularity conditions and when $x^\top \theta_*$ belongs to the interior of $\cU$ for every $x\in \cX$, $\hat\theta_t$ and $\theta_t$ are guaranteed to be such that $x^\top \hat\theta_t$ and $x^\top \theta_t$ both belong to the interior of $\cU$.%
\footnote{
To illustrate why regularity conditions are needed,
we borrow an example from page 153 of the book by \citet{barndorff2014information}.
Letting $Q$ be a Pareto distribution with shape parameter $\alpha>1$ and scale parameter of one and $\cX = \{1\}$, we have $U_Q = (-\infty,0]$, $\psi$ is left-differentiable at $0$ with a finite derivative of $1+1/(\alpha-1)$, while $\psi(0)=0$.
One can then show that if the average of observed values exceeds $1+1/(\alpha-1)$,
the maximum likelihood estimate of $u=\theta$ is $0\in \partial U_Q$.
}
Otherwise, this can always be achieved by modifying $-\log p(y;u)$ to return infinity for $u$ not in $\cU$.

\subsection{Why do linear perturbations work?}\label{sec:why-linear}
We now motivate our choice of linear perturbations. Our starting point is the observation that to guarantee the success of algorithms that choose actions greedily with respect to a model with a randomised parameter, it suffices that the random parameter is \emph{optimistic} in the sense that the mean reward for the best arm under the random parameter, say $\theta_t$, exceeds that under the true model parameter $\theta_\star$, with some probability bounded away from zero in a uniform manner \citep[for intuition, see Chapters 7 and 36 in ][]{lattimore2018bandit}.

A general approach to guarantee to achieve optimism is to sample the random parameters from a distribution that is nearly uniform over a set of the form
\begin{align*}
C(\Delta) = \{ \theta \in \Theta \,:\,  \ell_{t-1}(\theta) \le \ell_{t-1}(\hat\theta_{t-1}) + \Delta\}\,,
\end{align*}
where $\Delta>0$ is a tuning parameter (think `suitably large constant') and $\ell_{t-1}$ is the unregularised negative log-likelihood function corresponding to the observations available at the beginning of round $t$. For the purposes of this section,
for simplicity, we set $\lambda=0$ and consider the case when $\Theta=\R^d$.

Considering using a quadratic approximation to the negative log-likelihood function, we obtain the set
\begin{align*}
\widetilde C(\Delta) = \{ \theta \in \Theta \,:\,  \frac12\snorm{\theta-\hat{\theta}_{t-1}}^2_{F_{t-1}(\hat\theta_{t-1})} \le \Delta\}\,,
\end{align*}
where
\begin{align}
F_{t-1}(\theta) = \sum_{i=1}^{t-1} I(\phi_{X_i}(\theta)) \dot{\phi}_{X_i}(\theta) \dot{\phi}^\top_{X_i}(\theta) \label{eq:fdef}
\end{align}
is an approximation to the curvature (second derivative) of $\ell_{t-1}(\theta)$, and
corresponds to the Fisher-information underlying the parametric model $(P(dy;\phi_x(\theta)))_{\theta\in \Theta}$ with the design $X_1,\dots,X_{t-1}$. This approximation is based on the following observation.
\begin{proposition}\label{prop:fisher-info}
Let
$\ell(x,y;\theta) = -\log p(y; \phi_x(\theta))$ and
$\ell(x;\theta) =  \int \ell(x,y;\theta) P(dy; \phi_x(\theta_\star))$.
Then, under suitable regularity assumptions,
\begin{align*}
\frac{d^2}{d\theta^2} \ell(x,\theta_\star)  =  I(\phi_x(\theta_\star))\, \dot{\phi}_x(\theta_\star) \dot{\phi}_x^\top(\theta_\star)\,.
\end{align*}
\end{proposition}

We now argue that the said sampling goal can be accomplished by minimising a linearly perturbed negative log-likelihood function. For this let $g_i = \dot\phi_{x_i}(\hat\theta_{t-1})$ and
\begin{align*}
W = a\, \sum_{i=1}^{t-1} Z'_{t,i-1}  I(\phi_{X_i}(\hat\theta_{t-1}))^{1/2} g_i\,,
\end{align*}
where $Z'_{t,1},\dots,Z'_{t,t-1}$ are zero mean random variables that are independent of each other and of the past.
Let $\theta_t  = \argmin_{\theta} \ell_{t-1}(\theta) - W^\top \theta$ and $F=F_{t-1}(\hat\theta_{t-1})$.
Using  the first-order optimality condition
and if we also replace the gradient of $\ell_{t-1}$ with its local quadratic approximation, $\tilde \ell_{t-1}(\theta) = \ell_{t-1}(\hat\theta_{t-1}) + \frac12\snorm{\theta-\hat\theta_{t-1}}^2_{F}$,
 we get that
\begin{align*}
F^{1/2}(\theta_t -\hat\theta_{t-1}) \approx F^{-1/2} W\,.
\end{align*}
Now, note that $W \sim \cN(0, a^2 F)$, hence $F^{-1/2} W \sim \cN(0,a^2I_d)$.
It follows that, as long as the approximations used are precise enough,
given the past, the distribution of $F^{1/2}(\theta_t -\hat\theta_{t-1})$ is close to $\cN(0,a^2 I_d)$, hence, for $a$ sufficiently large,
$\theta_t$ is approximately uniform over $C(\Delta)$. Examining the perturbations in \cref{alg:EVILL}, we see these feature an additional term $a \lambda^{1/2} Z_t$ where $Z_t \sim \cN(0, I_d)$---this accounts for the effect of regularisation, which we omitted here.

The linear perturbation of the loss used by EVILL
is appealing from an implementation point of view:
to implement EVILL, one only needs to be able to construct a random vector $W$; since we expect that existing libraries that fit models can also deal with losses with the extra linear term, the implementation of EVILL should take only a few extra lines of code.

We provide an in-depth discussion of the computational complexity of EVILL in \cref{apx:computational}.

\subsection{Relation to Thompson sampling}\label{sec:TS-EVILL}
From the description of the previous section, it is clear that EVILL is a close relative of
direct parameter randomisation methods, such as Thompson (posterior) sampling.
To make this concrete, fix some (prior) distribution $\pi_0$ over $\Theta$ and let $p_{t-1}$ denote the posterior corresponding to $\pi_0$ and the likelihood functions $\theta \mapsto p( Y_i; \psi_{X_i}(\theta) )$ with $i\in\{1,\dots,t-1\}$.
Thus,
\begin{equation*}
 p_{t-1}(d\theta) \propto \pi_0(d\theta) \exp( - \cL(\theta;\cD_{t-1})).
\end{equation*}
Now note that for as long as $\theta$ is `close' to $\hat\theta_{t-1}$,
$\cL(\theta;\cD_{t-1})
\approx \cL(\hat\theta_{t-1};\cD_{t-1}) + \frac12\snorm{\theta-\hat\theta_{t-1}}_{F}^2=:\tilde \ell_{t-1}(\theta)$, where $F$ is as above.
For reasons that will become clear in a moment, let
$\pi_0= \cN(0,  \lambda^{-1} I)$.
Now, if most of the probability mass of $p_{t-1}$ concentrates in a small neighbourhood of $\hat\theta_{t-1}$, we get that the Gaussian distribution
\begin{equation*}
	\tilde p_{t-1}(d\theta) \propto \exp\left( -\frac\lambda2\snorm{\theta}^2  - \tilde\ell_{t-1}(\theta)\right)d\theta
\end{equation*}
is a good approximation to $p_{t-1}$.
The distribution $\tilde p_{t-1}$, which is based on the best local quadratic fit to the log-density of $p_{t-1}$ around the mode $\hat \theta_{t-1}$ of $p_{t-1}$, is known as the \emph{Laplace approximation} to $p_{t-1}$.

Observe that both $\tilde p_{t-1}$ and the sampling distribution induced by EVILL are based on a quadratic approximation to the negative log-likelihood. In fact, sampling from the Laplace approximation $\tilde p_{t-1}$ can also be implemented using a `sample-then-optimise' method  \citep{PaYu10,antoran2022sampling}, which can be shown to coincide with choosing the minimiser of $\tilde \ell_{t-1}(\theta)-W^\top \theta$; meanwhile, EVILL chooses the minimiser of $\ell_{t-1}(\theta)-W^\top \theta$.

Sampling from the Laplace approximation to the posterior has been used in many prior works, including \citet{chapelle2011empirical,russo2018tutorial,abeille2017linear,kveton2019randomized}, and is generally a computationally efficient approach. Yet, because the matrix $F$ changes with the parameter estimates, na\"ive approaches to sampling from the Laplace approximation scale poorly when either $d$ or $t$ is large---efficient implementations use the `doubling trick' or the aforementioned sample-then-optimise approach. In contrast, EVILL is a simple to implement alternative.

\section{RESULTS FOR GENERALISED LINEAR BANDITS}
\label{sec:theory}
In this section we consider EVILL in generalised linear bandits (GLBs) as defined in \cref{ex:glm}. We start with recalling the definition natural exponential families (NEFs). Next we show that when the reward distributions are given by a natural exponential family of distributions, EVILL and PHE with additive data perturbations are equivalent (\cref{prop:eq}), but that the equivalence may break outside this setting; indeed, we show that on a slightly enlarged class of bandits, PHE introduces a `bias' that may lead to linear regret, which is avoided by EVILL\@. We then introduce our main theoretical result, \cref{thm:main}, which gives a guarantee on the regret of EVILL that is comparable to the existing results for Thompson sampling based approaches. This puts EVILL---and thus, in the special case of generalised linear bandits, PHE---on equal footing with Thompson sampling. We end by highlighting a technical result, \cref{lem:mgf-bound}, that states that members of a self-concordant NEF distribution are sub-exponential---this allows us to remove the boundedness condition from the prior works of \textcite{abeille2021instance,russac2021self}.

\subsection{Natural exponential families} \label{sec:nef}
For any probability distribution $Q$ over the reals
and $u\in \R$, we let $\psi_Q(u) = \log \int e^{u y} Q(dy)$ be the \emph{cumulant-generating function} of $Q$, and let $\cU_Q \subset \R$ be the largest set  on which $\psi_Q(u)$ is finite. The natural exponential family generated by a distribution $Q$ is the family
$(Q_u(dy))_{u\in \cU_Q}$ of probability measures on $\R$ given by
\begin{equation*}
	Q_u(dy) = \exp(yu - \psi(u)) Q(dy) \spaced{for} u \in \cU_Q\,.
\end{equation*}
As is well known, $\cU_Q$ is always an interval.
We only consider \emph{regular} families, that is, families where
$\cU_Q^\circ$, the interior of $\cU_Q$, is not empty.

Our earlier examples, that is, the Bernoulli and Gaussian distributions are examples of regular families with $\cU=\cU_Q = \R$.
The Bernoulli distribution corresponds to the base distribution that is the uniform distribution on $\{0,1\}$, while the Gaussian distribution corresponds to $Q$ chosen as the standard normal distribution.
The latter might be considered a somewhat unusual choice (as opposed to choosing $Q$ to be Lebesgue measure), though this makes no difference.

\subsection{The equivalence of EVILL and PHE}
\label{sec:equiv}
In this section we show that in generalised linear bandits,
EVILL reduces to PHE with additive, data-dependent perturbations, and demonstrate how this reduction can fail outside that setting. For the first part, we work under the following assumption.
\begin{assumption}[Generalised linear bandit]
\label{ass:glb}
The family $(P(\cdot;u))_u$ that determines the rewards in our bandit instances is a natural exponential family generated by some base $Q$. That is,
\begin{align*}
P(\cdot;u) = Q_u(\cdot) \spaced{for} u\in \cU\,,
\end{align*}
where $\cU \subset \cU_Q^\circ$ is an interval with non-empty interior. Furthermore, we assume that the localisation maps are linear; that is, $\phi_x(\theta) = x^\top \theta$ for each $x \in \cX$ and $\theta \in \Theta$.
\end{assumption}

We now show how the PHE algorithm of \citet{kveton2019randomized} may be understood in terms of linear extensions of the likelihood function. Under \cref{ass:glb}, and choosing base measure $\nu = Q$, we get the density $p(y;u) = \exp(yu - \psi(u))$.
Recalling the notation $\ell(y;u) = -\log p(y;u)$, this gives
\begin{align*}
\ell(y;u) = \begin{cases}
 -yu + \psi(u)\,, & \text{if } y \in \supp(Q)\,;\\
 +\infty\,, & \text{otherwise}\,.
 \end{cases}
\end{align*}
Now consider the linear extension of $\ell$ to all of $\R$:
\begin{align*}
\tilde\ell(y;u) = -yu + \psi(u)\,, \qquad y\in \R\,.
\end{align*}
Using this function, for any $y\in \supp(Q)$ and any $z\in \R$, we have
\begin{align}
\tilde\ell(y+z;u)
& = \ell(y;u)-zu\,. \label{eq:lilo}
\end{align}
For $\cD = ((x_1,y_1),\dots,(x_s,y_s))$, $(x_i,y_i)\in \cX\times \R$,
let $$\tilde\cL(\theta;\cD) = \sum_{i=1}^s \tilde\ell(y_i; \phi_{x_i}(\theta)) + \frac{\lambda}{2} \snorm{\theta}^2\,.$$
Assume also that $y_1,\dots,y_s\in \supp(Q)$.
For $z\in \R^{s}$ let
\[
\cD^{z} = ((x_1,y_1+z_1),\dots,(x_s,y_s+z_{s}))\,.
\]
Then, from \cref{eq:lilo} and $\phi_{x_i}(\theta) = x_i^\top \theta$ we get
$$\tilde\cL(\theta;\cD^{z})
 = \cL(\theta;\cD) - \sum_{i=1}^{t-1}z_i x_i^\top \theta\,.$$
The PHE algorithm of \citet{kveton2019perturbed} for GLBs  computes at round $t \in \N^+$ a random parameter $\theta_t \in \argmin_{\theta}\tilde\cL(\theta; \cD^{Z_t}_{t-1})$, where $\cD^{Z_t}_{t-1}$ are the thus far collected observations and where $Z_t \sim \cN(0, a I_{t-1})$ is chosen independently of the past.
Then PHE chooses the action that is optimal under this random parameter $\theta_t$ for the action to be played in round $t$.

Comparing the above with EVILL, we can see that, in this GLB setting, EVILL is a variant of PHE provided that we replace the scaled identity covariance of $Z_t$ with a diagonal covariance that uses the Fisher information of the respective data-entries:
\begin{proposition}\label{prop:eq}
In generalised linear bandits, EVILL reduces to a variant of GLM-PHE where the perturbation vector used by PHE in round $t$ is
\begin{align*}
Z_{t}
 \sim \cN(0,a^2 \diag(I(X_1^\top \hat\theta_{t-1}),\dots,I(X_{t-1}^\top \hat\theta_{t-1}) ))\,.
\end{align*}
\end{proposition}

However, the equivalence of EVILL and PHE \emph{breaks in a strong sense} beyond generalised linear bandits. In particular, when the natural exponential family is replaced with a slightly bigger class, such as the exponential family, additive reward perturbations can lead to linear regret. Consider the following example.

\begin{example}[Rayleigh linear bandit]\label{ex:rayleigh}
	Let $\cU = [0, \infty)$ and for any $u \in \cU$, let $P(\cdot; u)$ be the Rayleigh distribution with parameter $u$, that is one with density with respect to $\nu(dy) =  2 y\mathbf{1}_{[0,\infty)}(y) m(dy)$ given by \[p(y;u) =  u \exp(-u y^2)\,,\] where $u\in \cU = [0,\infty)$, $\mathbf{1}_{[0,\infty)}$ is the characteristic function of $[0,\infty)$ and $m$ is the Lebesgue measure. Let $\phi_x(\theta) = x\tran\theta$, for all $x \in \cX$, $\theta \in \Theta$, with both $\cX$ and $\Theta$ subsets of the positive orthant of $\R^d$.
\end{example}

In the above Rayleigh linear bandit example, $\tilde\ell$, the linear extension of the negative log-likelihood induced by $p$ to the whole real line satisfies
\[\tilde\ell(y;u) = -\log(u)+u y^2,\quad y\in \R\,.\]
Now, for $Z\sim \cN(0,1)$, \[\E[\tilde\ell(y+Z;u)] = -\log(u)+u (y^2+1) \neq \tilde \ell(y;u)\,\] and so, unlike in generalised linear bandits, the additive-perturbation-based estimate is \emph{biased}. As we shall see in \cref{sec:experiments}, this bias may lead PHE to converge to choosing a suboptimal arm indefinitely. In contrast, EVILL does not incur such bias.

\subsection{Regret guarantees}
In this section we state our theoretical results in which we give an upper bound on the regret of EVILL\@.
The bounds are given for a broad subclass of generalised linear bandits, captured by the following assumptions.
\begin{restatable}{assumption}{assumptionMain}\label{ass:main}
The following hold:
\begin{enumerate}[(i)]
\item \label{ass:nef} The family $(P(\cdot;u))_{u\in \cU}$
corresponds to a regular natural exponential family $\cQ$
with some base $Q$, and $\cU=\cU_Q=\R$.

\item \label{ass:sg} For all $u\in \R$ and a known $L > 0$, $\dot\mu(u) \leq L$.

\item \label{ass:sc}
$\cQ$ is $M$-self-concordant with a known constant $M>0$:
\begin{equation*}
	|\ddot\mu(u)| \leq M \dot\mu(u) \spaced{for all} u \in \R\,.
\end{equation*}

\item \label{ass:armset} $\cX \subset B_2^d = \{ x\in \R^d\,;\, \snorm{x}_2\le 1\}$ is closed

\item \label{ass:parset} $\Theta \subset S\cdot B_2^d = \{ \theta\in \R^d\,:\, \snorm{\theta}_2\le S \}$ where $S>0$ is known.

\item \label{ass:linloc} The localisation maps are linear: $\phi_x(\theta) = x^\top \theta$
for all $x\in \cX$ and $\theta\in \Theta$.
\end{enumerate}
\end{restatable}

\begin{remark}
	Observe that in \cref{ex:llb,ex:lgb}, these assumptions are met. Indeed, for the logistic linear bandit, $M=1/4$ and $L=1/2$, and for a Gaussian linear bandit with unit variance, $M=0$ and $L=1$.
\end{remark}

Assumptions \eqref{ass:nef}--\eqref{ass:sg} and \eqref{ass:armset}--\eqref{ass:linloc}
essentially appeared in the work of
\citet{filippi2010parametric}, who calls bandit models where these conditions hold
generalised linear bandits (GLBs). The \emph{self-concordance property}, assumption \eqref{ass:sc}, was introduced to the GLB literature by \citet{russac2021self}, and ensures that $I$, the Fisher information map, does not change too rapidly. Given that the algorithm uses a preliminary estimate of $\theta_\star$ together with function $I$, it is natural that the results require the control of the smoothness of $I$. In particular, in natural exponential families, $I(u) = \dot\mu(u)$ for any $u\in \cU_Q^\circ$, and so this condition guarantees that $I(u)/I(v) \le e^{M|u-v|}$ for any $u,v\in \cU_Q$, which is what we will need in our proofs. An additional property of natural exponential families is that, there, $\dot\mu(u)$ is also the same as the variance of $Q_u$, and thus $L$ is simply an upper bound on the variance of $Q_u$ for any $u \in \R$.

With the assumptions in place, the main result of this section is the following bound on the regret of EVILL:
\begin{theorem}\label{thm:main}
Let \cref{ass:main} hold and let $n \in \Np$, the horizon, be given. Assume that, for some sufficiently small $b>0$, the prior observations satisfy
\[\max_{x \in \cX} \norm[V^{-1}_{\tau}]{x} \leq b\,,\]
where \[V_{\tau} = \sum_{i=1}^{\tau} X_i X_i\tran + \lambda I_d\,.\] Then, with appropriate choices of parameters $a,\lambda > 0$, with high probability, the regret of EVILL, \cref{alg:EVILL}, satisfies
\begin{equation*}
	R(n) \leq \tilde{O}(d^{3/2} \sqrt{n\dot\mu( x_\star^\top \theta_\star) }) + O( \tau )\,,
\end{equation*}
where $\tilde{O}$ hides problem dependent constants and the last term accounts for the
regret incurred during the $\tau$ steps during which the prior observations are obtained.
\end{theorem}

\begin{remark}
	When the prior observations are chosen by using a standard warm-up routine, discussed in \cref{apx:warm-up}, \[\tau=O((\log n)^2)\,,\] with constants as detailed in \cref{remark:warm-up} of \cref{apx:regret-formal}.
\end{remark}

When specialised to logistic bandits, our result is comparable to
that proven for the TS-ECOLog algorithm of \citet{faury2022jointly}, which is the state-of-the-art for randomised methods.
In particular, EVILL, just like TS-ECOLog, adapts to
$\dot\mu( x_\star^\top \theta_\star)$, the variance of rewards associated with the optimal arm. See \cref{apx:regret-formal} for a formal statement of this regret bound, its proof---which builds on the works of \citet{abeille2017linear,faury2020improved,russac2021self,faury2022jointly}---and more discussion.

\subsection{Confidence sets for NEFs}

Our proofs also extend the confidence sets developed for natural exponential family distributions by \citet{faury2020improved,russac2021self}, removing the assumption that rewards are uniformly bounded with probability one. This is based on the following moment generating function bound:

\begin{restatable}{lemma}{lemmaMGF}\label{lem:mgf-bound}
    Consider an $M$-self-concordant NEF $\cQ$.
    Then, for any $u\in \cU_Q^\circ$ and all $|s| \leq \log(2) / M$,
    \begin{equation*}
        \psi_{Q_u}(s) \leq s \mu(u) + s^2 \dot\mu(u)\,.
    \end{equation*}
\end{restatable}
Thus, $Q_u$ is what \citet{wainwright2019high} would describe as a \emph{sub-exponential distribution} with parameters
\[(\nu,\alpha)=(\sqrt{2\dot\mu(u)},\,M/\log(2))\,.\] The above bound and the resulting confidence sets are proven in \cref{apx:nef} and \cref{apx:confidence-sets} respectively.

\section{EXPERIMENTS}\label{sec:experiments}

We present two experiments.\footnote{Code: \url{https://github.com/DavidJanz/EVILL-code}} The first, on a synthetic logistic bandit problem, shows the importance of our data-dependent perturbations. The second, on a Rayleigh parameter estimation problem and a Rayleigh linear bandit problem, shows that PHE can suffer catastrophic bias outside the exponential family.

\begin{figure}[t]
    \flushleft
    \includegraphics[width=1.0\columnwidth]{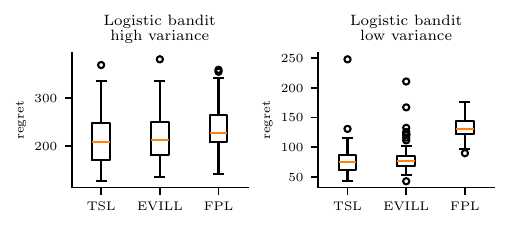}
    \caption{Regret of TSL, EVILL and FPL on two logistic linear bandit tasks: high variance, where $\dot\mu(x_\star\tran\theta_\star)$ is high, and low variance, where  $\dot\mu(x_\star\tran\theta_\star)$ is small. Box plots are based on the regret from 100 instances, and show the median and the interquartile (IQ) range, with whiskers restricted to $1.5\times$ the IQ range.}\label{fig:logistic}
\end{figure}

\begin{figure}[t]
    \flushleft
    \includegraphics[width=1.0\columnwidth]{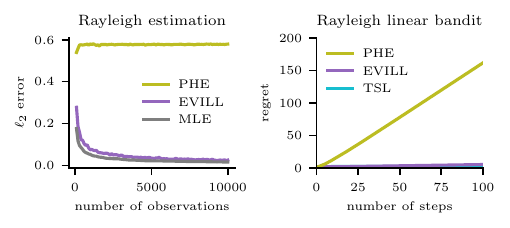}
    \caption{Rayleigh parameter estimation and bandit experiments. The left panel shows the $\ell_2$ error in estimating the parameter of a Rayleigh bandit when data is i.i.d., uncontrolled by the bandit algorithm.
    In addition to PHE and EVILL's estimates, the MLE estimate using the same data is also shown for
    reference.
	The right panel shows regret for TSL, PHE and EVILL for the same Rayleigh bandit problem, over $100$ steps of interaction. Both panels plot means over $100$ iterations and no confidence intervals---the latter were too tight to be visible. EVILL and TSL lines in right panel overlap.}
	\label{fig:rayleigh}
\end{figure}

\paragraph{Logistic bandit} Consider the classic logistic linear bandit setting, \cref{ex:llb}, with arm set $\cX$ consisting of all $x \in \{0,1\}^{10}$ such that exactly three components of $x$ are nonzero---real world problems often exhibit this kind of combinatorial arm structure, features can either be present, or not, with the choice of $10$ and $3$ being arbitrary. W test on two $\theta_\star$, one giving a high variance problem, with  $\dot\mu(x_\star\tran\theta_\star) \approx 0.15$, and a low variance problem, with $\dot\mu(x_\star \tran \theta_\star) \approx 0.02$. Specifically, we consider $\theta^\star$ of the form
\begin{equation*}
    [\theta_\star]_i = \frac{c_1}{(1 + i + c_2)^2}
\end{equation*}
for each $i \in [d]$, a simple inverse polynomial form, and set $c_1,c_2 \in \R$ using numerical optimisation to give the right variances, and satisfy $\min_{x \in \cX} x\tran\theta_\star = 0.1$.

As baselines, we use the Thompson sampling algorithm with Laplace approximation (TSL) and the follow-the-perturbed-leader (FPL) algorithms, both of \citet{kveton2019randomized}. FPL is the previous additive noise PHE method where the noise magnitude is not data-dependent. We set perturbation scale parameter $a$ to $a=1$ for TSL and EVILL (an arbitrary choice), and $a=1/2=\sqrt{L}$ for FPL, the latter such that the behaviour of FPL matches that of the other two algorithms in the high variance setting.

We provide all algorithms with a warm-up of $120$ (number of arms) observations, with the actions selected uniformly at random, and run for a total of $10,000$ iterations (long enough so that the regret of all algorithms is clearly sublinear), and run $100$ independent copies of the experiment, which suffices to make confidence intervals negligible. The results, shown in \cref{fig:logistic}, confirm that the regret of EVILL, alike that of TSL, scales favourably with  $\dot\mu(x_\star\tran\theta_\star)$. In contrast, FPL does not perform as well on the lower variance instance.

\paragraph{Rayleigh experiment} Our second experiment is based on the linear Rayleigh bandit of \cref{ex:rayleigh}, and consists of two parts: estimation, and the full bandit problem. We will use $\theta_\star = (0.9, 0.85)$ for both.

For the estimation part, we take $\cX = \{(1, 0), (0, 1) \}$ and construct a data set composed of $X_1, \dotsc, X_n$ sampled uniformly at random from $\cX$, and corresponding observations $Y_1, \dotsc, Y_n$ sampled from a Rayleigh distributions with parameters $X_i\tran\theta^\star$, $i \in [n]$.
The bandit algorithms PHE and EVILL are then used to produce parameter estimates with the modification
that their action choice for round $t$ is overwritten to be $X_t$.
This way, we can test in isolation whether they are able to produce good parameter estimates. In \cref{fig:rayleigh}, left panel, we plot the mean $\ell_2$-error of the estimates of $\theta_\star$ as a function of data subset size $n \in \{100, 200, \dotsc, 10,000\}$ for the two methods, as well as that of a maximum likelihood estimate, for reference. The results show that PHE is hopelessly inconsistent, while the EVILL estimates are only a little worse than those of the unperturbed MLE.

For the bandit experiment, we have arms $\cX = \{ (1, 0.99), (0.1, 0.05)\}$, chosen specifically such that the bias we expect PHE to exhibit leads to it choosing the suboptimal arm. We run TSL, PHE and EVILL for $100$ steps of interaction with no warm-up. All methods use the same noise scale, which was somewhat arbitrarily chosen to be $a=1.0$. As shown in the right panel of \cref{fig:rayleigh}, both EVILL and TSL suffer almost no regret, while PHE suffers linear regret---as expected.

\section{CONCLUSION}

The main contribution of this work suggests replacing perturbed-history exploration
with EVILL\@.
The main appeal of EVILL when compared to its alternatives is that
it can be implemented by adding a few extra lines to any code that implements
model fitting and policy optimisation, while we also expect to be
competitive with its alternatives. We have showed that this holds in generalised linear bandits.
Multiple intriguing avenues for further research, including  the interplay between computation, memory, and performance, remain to be explored. Another potential research direction involves evaluating EVILL's interaction with nonlinear models, such as neural networks. Drawing from experience gained with Thompson sampling, we also anticipate EVILL to be useful in reinforcement learning.

\section*{Acknowledgements}
Csaba Szepesv\'ari would like to express his gratitude to Branislav Kveton
and Benjamin van Roy
for their invaluable feedback during earlier stages of this research.
Csaba Szepesv\'ari also gratefully acknowledges funding  from the Canada CIFAR AI Chairs Program, Amii and NSERC.

\printbibliography
\onecolumn
\appendix
%!TEX root =  main.tex
%!TEX spellcheck = en_UK

\section{ON COMPUTATIONAL COMPLEXITY}\label{apx:computational}
A naive implementation of the algorithm stores all the data.
This uses $O(nd)$ memory.

Regarding the compute cost, first we note that the algorithm is tolerant to approximation errors in solving the optimisation problems. In particular, it follows with a simple argument that it is sufficient to optimise the losses in Steps~\ref{alg:mle} and \ref{alg:mlep} up to a constant accuracy, while it suffices to find an optimal arm up to an accuracy of $1/\sqrt{n}$.

In regard to the computational complexity of the various steps of the algorithm, the following holds:
In Steps~\ref{alg:mle} and \ref{alg:mlep}, the algorithm needs to solve a convex optimisation problem and the starting point of the design of our algorithm was that this can be done efficiently.
Indeed, the ellipsoidal algorithm can be used for this purpose with cost that scales polynomially in $n,d$ and $\log(1/\epsilon)$ \citep{Grotschel2011},
though we expect numerical algorithms tailored to the properties
of the specific natural exponential family underlying a specific bandit model to be more efficient.

Steps~\ref{alg:sample} and \ref{alg:weight} take $O(n (c+d))$ time, where $c$ is the cost of evaluating $\dot\phi_{X_i}$ at some input of its domain and the cost of evaluating the Fisher-information. For generalised linear models, $\dot\phi_{X_i}$ is the identity map, which means that $c$ represents just the cost of evaluating the Fisher-information. In ``named'' natural exponential families $c=O(1)$,
assuming that the standard transcendental functions can be evaluated in $O(1)$ time.

Another starting point for the design of our algorithm was that arm selection (Step~\ref{alg:arm}) can also be implemented in an efficient way.
When $\cX$ is finite, this can be done by looping through all the arms.
With this, the cost of selecting an arm is at most $O(|\cX| c)$ where $c$ represents the cost of evaluating $\mu$ and $\phi_x$. For generalised linear models, $\mu$ can be dropped (as it is an increasing function), and thus $c$ becomes the cost of evaluating $\phi_x$, which is $O(d)$.
When $\cX$ is a convex set (which one may assume in any case without loss of generality) and $\mu$ is an increasing function, the assumption that arm-selection can be done efficiently is equivalent to that linear optimisation over $\cX$ can be efficiently implemented.
This is known to be possible when $\cX$ is represented with either a membership or a separation oracle \citep{Grotschel2011}.

Putting everything together, we find that in generalised linear models
the total cost is polynomial in $n$ and $d$.
For a naive implementation, the cost, however, is at least quadratic in $n$.
It remains to be seen whether this quadratic cost can be avoided with \emph{any} randomised method that is not relying on maximising upper confidence bounds.

When compute time is a bottleneck, one approach is to parallelise the computation. Here, with $O(n)$ computers, accessing the same central storage, one can parallelise the computation of the perturbed estimates to reduce the per iteration compute time to $O(1)$ at the price of adding $O(n)$ memory per computer (compute $t$ would store $(Z_{t,i})_{i\le t}$). The speed-up is possible because the losses change by a little only between rounds, hence, computer $t$, which is used to get the weight vector to used in round $t$, can start the optimisation from the result that it obtained by optimising the objective just with the data available in rounds $1,\dots,t-2$.

\section{PRIOR OBSERVATIONS, AND THE BASICS OF SELF-CONCORDANCE}\label{apx:warm-up}

Recall that \cref{alg:EVILL} asks for prior observations $(X_1, Y_1), \dotsc, (X_\tau, Y_\tau)$. Then, \cref{thm:main} asks that, for $V_\tau = \sum_{i=1}^\tau X_i X_i\tran + \lambda I$ and some $b > 0$, the \emph{precondition}
\begin{equation*}
    \max_{x \in \cX} \norm[V_\tau^{-1}]{x} \leq b
\end{equation*}
holds. In this appendix, we look at what these prior observations buy us---which gives us an excuse to introduce some basic results on self-concordance---and how to gather these prior observations efficiently. We start with the latter.

\subsection{Warm-up procedures}

Let $X_1,X_2,\dots\in \R^d$ and for $t\ge 1$ let
$V_t = \lambda I + \sum_{s\le t} X_s X_s^\top$ with $V_0 = \lambda I$.
Suppose we want to produce prior observations $X_1,X_2,\dots,X_{\tau}$ so that
$\max_{x \in \cX} \norm[V_\tau^{-1}]{x} \leq b$.
We are interested in stopping as early as possible.
Hence,
in round $t+1$ we stop if $\max_{x \in \cX} \norm[V_t^{-1}]{x}$, while if this
is not satisfied we choose $X_{t+1}$ in some way.
Define
\begin{equation*}
    \tau = \min \{t \in [n] \colon \max_{x \in \cX} \norm[V_t^{-1}]{x} \leq b \}
\end{equation*}
as the index of the round when we stop.
How small can we make $\tau$ (or $\EE{\tau}$ when the $(X_t)$ are chosen by using a randomised method).

A slight modification of
Exercise~19.3 of \citet{lattimore2018bandit} gives the following:
\begin{lemma}[Elliptical potentials: You cannot have more than $O(d)$ big intervals]\label{ex:eliptical}
Let $V_0 = \lambda I$ and
$a_1,\ldots,a_n \in \R^d$ be a sequence of vectors with $\snorm{a_t}_2 \leq L$ for all $t \in [n]$. Further, let $V_t = V_0 + \sum_{s=1}^t a_s a_s^\top$. Then, the number of
times $\snorm{a_t}_{V_{t-1}^{-1}} \geq b$ is at most
\begin{align*}
\frac{3d}{\log(1+b^2)} \log\left(1 + \frac{L^2}{\lambda \log(1+b^2)}\right)\,.
\end{align*}
\end{lemma}
For completeness, we include the proof:
\begin{proof}
Let $\cT$ be the set of rounds $t$ when $\snorm{a_t}_{V_{t-1}^{-1}} \geq  b$ and $G_t = V_0 + \sum_{s=1}^t \one{s\in \cT} a_s a_s^\top$. Then
\begin{align*}
\left(\frac{d\lambda  + |\cT| L^2}{d}\right)^d
&\geq \left(\frac{\trace(G_n)}{d}\right)^d \\
&\geq \det(G_n)  \tag{determinant-trace inequality}\\
&= \det(V_0) \prod_{t \in \cT} (1 + \snorm{a_t}_{G_{t-1}^{-1}}^2)  \tag{expanding determinant}\\
&\geq \det(V_0) \prod_{t \in \cT} (1 + \snorm{a_t}_{V_{t-1}^{-1}}^2) \tag{dropping indicators}\\
&\geq \lambda^d (1+b^2)^{|\cT|}\,. \tag{definition of $\cT$}
\end{align*}
Rearranging and taking the logarithm shows that
\begin{align*}
|\cT| \leq \frac{d}{\log(1+b^2)} \log\left(1 + \frac{|\cT| L^2}{d\lambda}\right)\,.
\end{align*}
Abbreviate $x = d/\log(1+b^2)$ and $y = L^2/d \lambda$, which are both positive. Then
\begin{align*}
x \log\left(1 + y (3x \log(1+xy))\right)
\leq x \log\left(1 + 3x^2y^2\right)
\leq x \log((1 + xy)^3 )
= 3x \log(1 + xy)\,.
\end{align*}
Since $z - x \log(1 + yz)$ is decreasing for $z \geq 3x \log(1 + xy)$ it follows that
\begin{align*}
|\cT| \leq 3x \log(1 + xy) = \frac{3d}{\log(1+b^2)} \log\left(1 + \frac{L^2}{\lambda \log(1+b^2)}\right)\,.
\end{align*}
\end{proof}

Now, by \cref{ass:main}\eqref{ass:armset}, we can set $L=1$. It follows that regardless the choice of $X_{t+1}$, the total number of times $ \norm[V_t^{-1}]{X_{t+1}}\ge b$ holds for $t=0,1,\dots$ is bounded by
\begin{align*}
\tau_{\mathrm{naive}} = \frac{3d}{\log(1+b^2)} \log\left(  1+ \frac{1}{\lambda \log(1+b^2)} \right)\,.
\end{align*}
Thus, we can simply choose $X_{t+1} = \argmax_{x\in \cX} \norm[V_t^{-1}]{x}$, which guarantees that
$\tau\le \tau_{\mathrm{naive}}$.

The dependence on $b$ can be slightly improved by a refined approach.
To see how this can be done we need to recall the Kiefer-Wolfowitz theorem and some definitions.
For any distribution $\pi$ over $\cX$, let
\begin{equation*}
    \bar V(\pi) =  \int x x\tran \pi(dx) \spaced{and} g(\pi) = \max_{x \in \cX} \norm[\bar V(\pi)^{-1}]{x}^2.
\end{equation*}
The Kiefer-Wolfowitz theorem (e.g., Theorem 21.1 in \citet{lattimore2018bandit}) states that when $\cX$ is a compact subset of $\R^d$, there exists a probability distribution $\pi_\star$ over $\cX$ such that $g(\pi_\star) = d$, and the support of $\pi_\star$ is a finite set with at most $d(d+1)/2$-many points in it. With that, sampling $X_1, X_2, \dotsc$ independently of $\pi_\star$ gives us a stopping time $\tau$ that is, with high probability, on the order of $d / b^2$.
Alternatively, one can use $\lceil n \pi_\star(y) \rceil$ observations from $y$ in the support of $\pi_\star$ with $n=\lceil d/b^2 \rceil \ge d/b^2$ for a total of at most $d(d+1)/2+n \le d(d+1)/2 + 1 + d/b^2$ observations.
The distribution $\pi_\star$ is called the G-optimal design, for `globally optimal'.
This improves the dependency of $\tau$ on $b$ from slightly worse than $3 d/\log(1+b^2)$ to $d/b^2$.

It remains to be seen how a G-optimal design can be found.
Finding the G-optimal design, precisely, or up to a tiny approximation error, can be expensive.
Fortunately, there is no need for doing this.

In particular, finding a distribution $\pi$, such that, say, $g(\pi) \leq 2 g(\pi_\star)$, can be done in time almost linear in $d$, and using such a distribution in place of $\pi_\star$ increases the resulting stopping time $\tau$ by at most a factor of $2$. Specifically, when  $\argmax_{x\in \cX} c^\top x$ and $\argmax_{x\in \cX} \norm[V^{-1}]{x}^2$ can both be efficiently computed for any $c\in \R^d$ and positive definite $V$, an appropriate distribution $\pi$ can be found in $O(d \log \log d)$ time using the classical Frank-Wolfe algorithm (see \citet{lattimore2018bandit}, Note 3 of section 21.2, together with Algorithm 3.3 and Theorem 3.9 of \citet{Todd2016}). Furthermore, the size of the support of the resulting distribution $\pi$ will match the number of steps that algorithm, $O(d \log \log d)$.
When only a linear optimisation oracle over $\cX$ is available, \citet{hazan2016volumetric} gives an alternate poly-time approach. As mentioned earlier, for $\cX$ convex, the ellipsoidal algorithm can be used at a polynomial cost for implementing both
$\argmax_{x\in \cX} c^\top x$ and $\argmax_{x\in \cX} \norm[V^{-1}]{x}^2$ \citep{Grotschel2011}.

\subsection{Prior observations and self-concordance: what do our prior observations guarantee?}\label{apx:warm-up-goal}

Per \cref{prop:fisher-info}, for an observation at location $x \in \cX$, we might like to introduce perturbations with variance that scales with $I(x\tran\theta_\star) = \dot\mu(x\tran\theta_\star)$, but we do not know $\theta_\star$. We have only some estimate $\hat\theta_{t-1}$, and use $\dot\mu(x\tran\hat\theta_{t-1})$ instead. Self-concordance lets us bound how close $\dot\mu(x\tran\theta_\star)$ and $\dot\mu(x\tran\hat\theta_t)$ are.
For this, the following result, shown in \textcite{sun2019generalized}, will be useful:

\begin{lemma}\label{claim:dot-mu-exp}
    Consider an $M$-self-concordant NEF with domain $\cU_Q$. Then, for any $u, u' \in \cU^\circ_Q$,
    \[
    \dot\mu(u) \leq \dot\mu(u') e^{M|u - u'|}\,.
    \]
\end{lemma}

With that, if we know that, say,
\begin{align}
\max_{x \in \cX} |x\tran (\theta_\star - \hat\theta_{t-1})| \leq \frac{1}{M}\,,
\label{eq:accforip}
\end{align}
then it follows that $\dot\mu(x\tran\hat\theta_{t-1})$ is an $e$-multiplicative approximation to $\dot\mu(x\tran\theta_\star)$. Our requirements on the prior observations will be so as to make this happen---and slightly more, in that we actually end up needing that
 not just $\dot\mu(x\tran\hat\theta_{t-1})$ but
$\dot\mu(x\tran\theta_t)$ is also close to $\dot\mu(x\tran\theta_\star)$.
For this, it will suffice to slightly tighten \cref{eq:accforip}, but for now, to keep things simple, we keep considering only \cref{eq:accforip}.

To connect \cref{eq:accforip} to the prior observations, we first upper bound the right-hand side of \cref{eq:accforip}.
For this, for $t \in [n]$, let
\begin{align}
H_{t} = \sum_{i=1}^t \dot\mu(X_i\tran \theta_\star) X_iX_i\tran + \lambda I
\label{eq:htdef}
\end{align}
(and $H_0 = \lambda I$) and observe that since these matrices are positive definite, by Cauchy-Schwarz,
\begin{equation*}
    |x\tran(\theta_\star - \hat\theta_{t-1})| \leq \norm[H_{t-1}^{-1}]{x} \norm[H_{t-1}]{\theta_\star - \hat\theta_{t-1}}.
\end{equation*}
Thus, if we have a high probability upper bound on $\norm[H_{t-1}]{\theta_\star - \hat\theta_{t-1}}$ that holds for all $t \in [n]$, then it suffices that we ensure that, for all $t > \tau$, $\norm[H_t^{-1}]{x}$ times that upper bound is less than $1/M$. Since our warm-up is in terms of $V_t$ and not $H_t$, we need a way to convert between these. To do so, we define the constant $\kappa \ge 1$ as follows.

\begin{definition}[Constant $\kappa$]\label{def:kappa}
    Let $\kappa$ be the smallest upper bound on $\{1\} \cup \{1/\dot\mu(u)\,:\, |u| \leq S\}$:
    \begin{align*}
	\kappa = 1 \vee \max_{u: |u| \le S} \frac{1}{\dot\mu(u)}\,.
	\end{align*}
\end{definition}

With that, we have the bound $H_{t-1}^{-1} \preceq  \kappa V_{t-1}^{-1}$, from which the following claim follows immediately. Here, and in what follows, we use $\preceq$ to denote the Loewner partial ordering of positive semidefinite matrices (i.e., $A\preceq B$ if $B-A$ is positive semidefinite).

\begin{claim}[Precondition]\label{claim:precondition}
    For all $t \in \{\tau+1, \dotsc, n\}$, \[\max_{x \in \cX} \norm[H^{-1}_{t-1}]{x} \leq b\sqrt{\kappa}\,.\]
\end{claim}

It remains to establish just how small we need to set $b$, which will depend on how well we can upper bound $\norm[H_{t-1}]{\theta_\star - \hat\theta_{t-1}}$ (and $\norm[H_{t-1}]{\theta_\star - \theta_{t}}$, but that will not be hard given the previous bound). We return to this after establishing some additional definitions.

\section{PRELIMINARIES FOR THE CONSTRUCTION OF CONFIDENCE SETS AND THE REGRET BOUND}\label{apx:defs}

In this appendix, we list some standard properties of NEFs, define some quantities related to the likelihood function and use self-concordance results to provide bounds for these. Quantities defined here will be used throughout the remainder of the appendices.

\subsection{Natural exponential families and their cumulant generating functions}\label{apx:nef}

Consider a natural exponential family $\cQ$ with base $Q$, cumulant generating function $\psi(s) = \log \int e^{sx} Q(dx)$ and domain $\cU_Q$. The cumulant generating function $\psi$ is analytic on $\cU_Q^\circ$ (see Exercise 5.9 in \citet{lattimore2018bandit}), and, for any $u \in \cU_Q^\circ$, its derivatives are given by
\begin{align*}
    \dot\psi(u) &= \int x P(dx; u) = \mu(u)\\
    \ddot\psi(u) &= \int (x - \mu(u))^2 P(dx; u) = \dot\mu(u)\\
    \dddot\psi(u) &= \int (x-\mu(u))^3 P(dx; u) = \ddot\mu(u).
\end{align*}
That is, the derivatives of $\psi$ at $u$ give the first three central moments of the distribution $P(\cdot;u)$. For these and more properties of NEFs, see \citet{morris2009unifying}. For even more, see \citet{morris1983natural,letac1990natural}.

With these properties fresh in mind, we now restate and quickly prove \cref{lem:mgf-bound}.

\lemmaMGF*
\begin{proof}
	Writing out the expressions involved, it is immediate that $\psi_{Q_u}(s) = \psi(u+s) - \psi(u)$.
	Taking a second order Taylor expansion of $\psi(u+s)$ about $u$, assuming without loss of generality that $s>0$, we thus see that there exists a $\xi \in [u, u+s]$ such that
    \begin{equation*}
        \psi_{Q_u}(s)
        = s\dot\psi(u) + \frac{s^2 \ddot{\psi}(\xi)}{2}
        = s\mu(u) + \frac{s^2 \dot\mu(\xi)}{2},
    \end{equation*}
    where we used that $\dot\psi = \mu$. By self-concordance, that is, using \cref{claim:dot-mu-exp},
    and considering only $s$ that satisfy $|s| < \log(2) / M$, we have that
    \begin{align*}
         \dot\mu(\xi)
          \leq \exp(M |\xi - u|)\dot\mu(u)
         \leq \exp(M |s|)\dot\mu(u)  \leq 2\dot\mu(u),
    \end{align*}
    completing the proof.
\end{proof}

\subsection{The negative log-likelihood function of natural exponential families}

Recall from \cref{eq:nll} that the $\lambda$-regularised negative log-likelihood function given $\cD_t = ((X_i,Y_i))_{i=1}^t$ is
\begin{align*}
	\cL(\theta; \cD_t) = \frac{\lambda}{2} \|\theta\|^2 -\sum_{i=1}^t \log p(Y_i; X_i\tran\theta)\,.
\end{align*}
For an NEF with base $Q$, defining the density $p$ in the above with respect to $Q$, we have that $p(y; u) = \exp( yu -\psi(u))$, and so
\begin{align*}
	\cL(\theta; \cD_t) = \frac{\lambda}{2} \|\theta\|^2 -
	\sum_{i=1}^t  (Y_i X_i\tran\theta - \psi(X_i\tran\theta))\,,
\end{align*}
where recall that $\psi$ is the cumulant-generating function of $Q$. Since we will be interested in the global minimisers of $\cL(\theta; \cD_t)$, and $\psi$ is known to be strictly convex and smooth,
$\cL$ has a unique minimiser, which is the unique stationary point of $\cL$.
Taking the derivative of both sides of the previous display
with respect to $\theta$, and using that $\dot\psi = \mu$, we get
\begin{equation}
	\left(\frac{\partial}{\partial \theta} \cL(\theta;\cD_t)\right)\tran =
	\underbrace{\sum_{i=1}^t  \mu(X_i\tran \theta) X_i + \lambda \theta }_{=:g_t(\theta)}
	 - \sum_{i=1}^t X_i Y_i\,, \label{eq:gdef}
\end{equation}
where the previous display defines the `gradient function' $g_t$.
It follows that  $\hat\theta_t = \argmin_{\theta\in \R^d} \cL(\theta;\cD_t)$ satisfies
\begin{align}
g_t(\hat\theta_t) = \sum_{i=1}^t X_i Y_i\,.
\label{eq:basicopt}
\end{align}

We will also need to reason about the Hessian
$H(\theta; \cD_t) = \frac{\partial^2}{\partial \theta^2} \cL(\theta;\cD_t)$
of $\cL(\theta; \cD_t)$. Further differentiation gives
\begin{equation*}
    H(\theta; \cD_t) = \sum_{i=1}^t \dot\mu(X_i\tran\theta) X_iX_i\tran + \lambda I.
\end{equation*}
Now, notice that
\begin{align}
H_t =  H(\theta_\star; \cD_t)\,.
\label{eq:htdefH}
\end{align}
We further let
\begin{align}
\hH_t = H(\hat\theta_t, \cD_t)
\label{eq:httdef}
\end{align}
The reader will be reminded of these definitions frequently.

\subsection{The secant gradient approximation, the Hessian, and self-concordance}

We let $\alpha: \cU^\circ_Q \times \cU^\circ_Q \to \R$ denote a `secant approximation' to $\dot\mu$.  In particular, for any $u, u' \in \cU^\circ_Q$,
\begin{equation*}
    \alpha(u,u') =
    \begin{cases}
    \dot\mu(u)\,, & \text{ if } u=u'\,;\\
    \frac{\mu(u)-\mu(u')}{u-u'}\,, & \text{otherwise}\,.
    \end{cases}
\end{equation*}
For $u\ne u'$, the value of $\alpha(u, u')$ is the gradient of the secant line from $u$ to $u'$ and $\lim_{u'\to u} \alpha(u,u') =\dot\mu(u)$. Furthermore, clearly, $\alpha$ is a symmetric function of its arguments:
\begin{align*}
\alpha(u,u') = \alpha(u',u) \text{ for all } u,u'\in \cU^\circ_Q\,.
\end{align*}
With self-concordance, $\alpha$ can be a relatively good approximation to $\dot\mu$, as shown by the following result, a consequence of \cref{claim:dot-mu-exp} combined with the fundamental theorem of calculus.
This result is a special case of Corollary 2 of \citet{sun2019generalized}, hence the proof is omitted.
\begin{claim}\label{lem:alpha-bound}
    For all $u, u' \in \cU^\circ_Q$,
    \begin{equation*}
        \dot\mu(u)
        h(-M|u-u'|)  \leq \alpha(u, u') \leq \dot\mu(u) h(M|u-u'|)\,,
    \end{equation*}
    where for $x\ne 0$,
    $h(x) =\frac{e^{x}-1}{x}$ and $h(0)=1$.
\end{claim}

Based on the secant approximation to the gradient, we construct the following approximation to the Hessian $H(\theta; \cD_t)$:
\begin{equation*}
    G(\theta, \theta'; \cD_t) = \sum_{i=1}^t \alpha(X_i\tran\theta, X_i\tran\theta')X_iX_i\tran + \lambda I.
\end{equation*}
Thanks to $\alpha$ being symmetric, $G$ is also a symmetric function of $\theta$ and $\theta'$.

By \cref{lem:alpha-bound}, it is clear that when $\theta'$ is close, in a suitable sense, to $\theta$, $G(\theta, \theta'; \cD_t) \approx H(\theta; \cD_t)$. The suitable sense here is captured by the pseudo-norm
\begin{equation*}
    D(v) = \max_{x \in \cX} |x\tran v|, \qquad v\in \R^d\,.
\end{equation*}
Indeed, \cref{lem:alpha-bound} with a crude bound gives the following upper bound on $H(\theta;\cD_t)$ in terms of $G(\theta,\theta';\cD_t)$:
\begin{claim}\label{claim:GH-bound}
    For any $\theta,\theta' \in \mathbb{R}^d$,
     $ H(\theta; \cD_t) \preceq (1 + M D(\theta-\theta')) G(\theta,\theta'; \cD_t)
     = (1 + M D(\theta-\theta')) G(\theta',\theta; \cD_t)
     $.
\end{claim}
\begin{proof}
Just notice that for $h$ from  \cref{lem:alpha-bound} it holds that for $x\ge 0$, $h(-x) \ge 1/(1+x)$. This together with the definition of $D$ and \cref{lem:alpha-bound} gives the result.
\end{proof}
We can, of course, get a two-sided bound from \cref{lem:alpha-bound}, but we happen not to need it.
Regarding the pseudonorm $D$, clearly, for any $v\in \R^d$,
\begin{align}
D(v) \le \norm{v}
\label{eq:dbound}
\end{align}
holds thanks to $\cX \subset B_2^d$.
We will often rely on this inequality.

Finally, directly from the definitions of $G(\theta,\theta';\cD_t)$ and $g_t$, we see the following.
\begin{claim}\label{claim:g-t-mvt}
    For any $\theta, \theta'\in \R^d$, $g_t(\theta) - g_t(\theta')=G(\theta, \theta'; \cD_t)(\theta-\theta')
    =G(\theta', \theta; \cD_t)(\theta-\theta')
    $.
\end{claim}
We will use these when working with first order optimality conditions for the negative log-likelihood function.

\section{CONFIDENCE SETS FOR THE REGULARISED MLE IN SELF-CONCORDANT NATURAL EXPONENTIAL FAMILIES}\label{apx:confidence-sets}

In this appendix, we show the following concentration result.

\begin{lemma}\label{lem:confidence-sets}
    For any $\delta,\lambda > 0$, let $\gamma_t(\delta,\lambda)$ be given by
    \begin{equation*}
        \gamma_t(\delta, \lambda)
        = \sqrt{\lambda}\left(\frac{1}{2M} + S\right) + \frac{2Md}{\sqrt{\lambda}}\left(1 + \frac{1}{2}\log\left(1+\frac{tL}{\lambda d}\right)\right) + \frac{2M}{\sqrt{\lambda}}\log(1/\delta),
    \end{equation*}
    and let $\cE(\delta, \lambda)$ be the event
    \begin{equation*}
        \cE(\delta, \lambda) = \left\{\forall t \geq 0,\ \norm[H^{-1}_t]{g_t(\hat\theta_t) - g_t(\theta_\star)} \leq \gamma_t(\delta, \lambda)\right\}.
    \end{equation*}
    Then, under \cref{ass:main}, $\P(\cE(\delta, \lambda)) \geq 1-\delta$.
\end{lemma}
The above lemma is a direct consequence of \cref{lem:mgf-bound}, the bound on the cumulant generating function of a self-concordant natural exponential family distribution, which we just proved in \cref{apx:nef}.
We start the proof with an intermediate result in the next section.

\subsection{An intermediate concentration result}\label{apx:abstract-confidence-sets}

The following concentration inequality is a generalisation of Theorem 1 of \citet{faury2020improved}, who prove a similar statement under the assumption that each response $Y_t$ is absolutely bounded by 1, that is $|Y_t| \leq 1$ almost surely for all $t \in \N^+$. We relax this to a condition that bounds the conditional cumulant generating function of $Y_t$ by a quadratic term in a neighbourhood of zero.
Comparing the condition in the statement of the theorem with our \cref{lem:mgf-bound}, it is clear that, under \cref{ass:main}, we will be able to satisfy this condition.

\begin{theorem}\label{thm:novel-concentration}
    Fix $\lambda, M >0$.
    Let $(X_t)_{t\in \N^+}$ be a $B^d_2$-valued random sequence,
     $(Y_t)_{t \in \N^+}$ a real valued random sequence,
     $(\nu_t)_{t\in \N}$ be a nonnegative valued random sequence.
     Let $\F' = (\F'_t)_{t \in \N}$ be a filtration
    such that {\em (i)}
    $(X_t)_{t\in \N^+}$ is $\F'$-predictable
    and
    {\em (ii)}
    $(Y_t)_{t\in \N^+}$ and $(\nu_t)_{t\in \N}$ are $\F'$-adapted.\footnote{Thus, for $t\ge 1$, $X_t$ is $\F'_{t-1}$-measurable,
    $Y_t$, $\nu_t$ are $\F'_t$-measurable, while $\nu_0$ is also $\F'_0$-measurable.}
    Let $\epsilon_t = Y_t - \E[Y_t \mid \F'_{t-1}]$ and assume that the following condition holds:
    \begin{align}
        \E{[\exp(s \epsilon_t)\mid \F'_{t-1}]} &\leq \exp(s^2\nu_{t-1}) \spaced{for all} |s|\le 1/M \text{ and } t\in \N^+\,.
        \label{eq:lvbound}
    \end{align}
    Then, for $\widetilde H_t = \sum_{i=1}^{t} \nu_{i-1} X_iX_i\tran + \lambda I$ and $S_t = \sum_{i=1}^{t} \epsilon_{i}X_i$ and any $\delta>0$,
    \begin{align*}
        \P\left( \exists t \in \N^+ \colon
        \norm[\widetilde H_t^{-1}]{S_t}
        \geq \frac{\sqrt{\lambda}}{2M}
            + \frac{2M}{\sqrt{\lambda}}\log \left(\frac{\det(\widetilde H_t)^{1/2}\lambda^{-d/2}}{\delta}\right)
            + \frac{2M}{\sqrt{\lambda}}d\log(2)\right)
            \leq \delta.
    \end{align*}
\end{theorem}

The proof of \cref{thm:novel-concentration} follows identically to that of Theorem 1 in \citet{faury2020improved}, once their Lemma 5 is replaced with the following lemma. As the proof of their Theorem 1 is rather tedious, we do not reproduce it.

\begin{lemma}\label{lem:nonneg-supermartingale-window}
    Assume the conditions of \cref{thm:novel-concentration} hold.
    For each $t\in \N^+$, let $\bar{H}_t = \sum_{i=1}^{t}\nu_{i-1} X_iX_i\tran $ and $S_t = \sum_{i=1}^{t} \epsilon_i X_i$. For any $\xi \in B_2^d$, define the real-valued process starting with $M_0(\xi) = 1$ and given by
    \begin{align*}
        M_t(\xi) = \exp\left(\frac{1}{M}\xi\tran S_t - \frac{1}{M^2}\norm[\bar{H}_t]{\xi}^2\right) \spaced{for} t\in \N^+.
    \end{align*}
    Then, $(M_t(\xi))_{t=0}^\infty$ is a nonnegative supermartingale with respect to $\F'$.
\end{lemma}

\begin{remark}
    We use the notation $M_0(\xi), M_1(\xi), \dotsc$ for the process to match that of Theorem 1 and Lemma 5 of \citet{faury2020improved}. This process is not to be confused with the constant $M > 0$, which while only an abstract constant here, will of course become the self-concordance constant when we use the result.
\end{remark}

In the next and subsequent proofs the statements concerning conditional expectations hold almost surely, even when this is not explicitly stated.

\begin{proof}
    Fix $t\in\N^+$, $\xi\in B_2^d$.
    From \cref{eq:lvbound}, we get that for any $\F'_{t-1}$-measurable random variable $R_t$
    such that
     $|R_t|\le 1/M$, the inequality
    \begin{equation*}
		\E\left[\exp\left(R_t \epsilon_t  - R_t^2 \nu_{t-1}\right) \mid \F'_{t-1}\right] \le 1
    \end{equation*}
    holds.
    Define now $R_t = \xi^\top X_t/M$. Since, by assumption, $\xi,X_t\in B_2^d$, it holds that
    $|\xi^\top X_t|\le 1$ and hence
    $|R_t|\le 1/M$.
    Further, since $(X_t)_{t\in \N^+}$ is $\F'$-predictable, $X_t$ and hence also $R_t$ is $\F'_{t-1}$-measurable.
    Thus, by the previous inequality,
    \begin{align}
    \label{eq:tb}
	\E\left[\exp\left(\frac{\epsilon_t}{M}\xi^\top  X_t - \frac{\nu_{t-1}}{M^2}\xi\tran X_tX_t\tran \xi \right) \mid \F'_{t-1}\right] \le 1\,.
	\end{align}
	It then follows that
    \begin{align*}
        \E[ M_t(\xi) \mid \F'_{t-1}]
        &= \E\left[\exp\left(\frac{1}{M}\xi\tran S_t - \frac{1}{M^2}\norm[\bar{H}_{t}]{\xi}^2 \right) \mid \F'_{t-1}\right] \\
        &= \exp\left(\frac{1}{M}\xi\tran  S_{t-1} - \frac{1}{M^2}\norm[\bar{H}_{t-1}]{\xi}^2 \right)
        \E\left[\exp\left(\frac{\epsilon_t}{M}\xi^\top  X_t - \frac{\nu_{t-1}}{M^2}\xi\tran X_tX_t\tran \xi \right) \mid \F'_{t-1}\right]\\
        &\leq \exp\left(\frac{1}{M}\xi\tran  S_{t-1} - \frac{1}{M^2}\norm[\bar{H}_{t-1}]{\xi}^2\right) \tag{by \cref{eq:tb}}\\
        &= M_{t-1}(\xi)\,,
    \end{align*}
finishing the proof.
\end{proof}

\begin{remark}
As noted beforehand, the lemma just proved corresponds to that of Lemma~5 of \citet{faury2020improved}.
\citet{faury2020improved} prove this lemma based on their Lemma 7, which  states that for a centred random variable $\epsilon$ almost surely absolutely bounded by $1$, for $|s|\le 1$, $\log \E[\exp(s \epsilon)] \le \log(1+s^2 \Var[\epsilon])$. Noting that in their proof of Lemma~5 they use $\log(1+s^2 \Var[\epsilon]) \le s^2 \Var[\epsilon]$, we see that our proof simply replaces Lemma~7 with
requiring $\log \E[\exp(s \epsilon)]  \le s^2 \Var[\epsilon]$ directly (cf.  \cref{eq:lvbound}), which, ultimately will follow
from \cref{lem:mgf-bound}.
\end{remark}

\subsection{Proof of the confidence set result, \cref{lem:confidence-sets}}

We can now prove \cref{lem:confidence-sets}. This is effectively just combining \cref{ass:main}, \cref{lem:mgf-bound} and \cref{thm:novel-concentration}.

\begin{proof}[Proof of \cref{lem:confidence-sets}]
    Consider an arbitrary $t \in \N$.
	By first order optimality and the definition of $g_t$, the following inequality, which we copy from \cref{eq:basicopt} for the convenience of the reader, holds:
    \begin{equation*}
        g_t(\hat\theta_t)= \sum_{i=1}^t Y_iX_i\,.
    \end{equation*}
    Also, by the definition of $g_t$,
    \begin{align*}
     g_t(\theta_\star) = \sum_{i=1}^t \mu(X_i\tran\theta_\star) X_i + \lambda\theta_\star.
	\end{align*}
    Writing $\epsilon_i = Y_i - \mu(X_i\tran\theta_\star)$ and $S_t = \sum_{i=1}^t \epsilon_i X_i$, we thus have that
    \begin{equation*}
        g_t(\hat\theta_t) - g_t(\theta_\star) = S_t - \lambda\theta_\star.
    \end{equation*}
    Taking the $H_t^{-1}$-weighted 2-norm of the above and using the triangle inequality and recalling that $H_t \succeq \lambda I$ and $\norm{\theta_\star} \leq S$, we get
    \begin{equation*}
        \norm[H_t^{-1}]{g_t(\hat\theta_t) - g_t(\theta_\star)}
        \leq
        \norm[H_t^{-1}]{S_t}
        + \lambda \norm[H_t^{-1}]{\theta_\star}
        \leq \norm[H_t^{-1}]{S_t} + \sqrt{\lambda} S.
    \end{equation*}

    We complete the proof by bounding $\smash{\norm[H_t^{-1}]{S_t}}$
    using \cref{thm:novel-concentration} applied to $(X_t,Y_t)_{t\in\N^+}$ and $(\nu_t)_{t\in \N}$, where the latter is defined by
    \begin{align}
    \label{eq:varid}
	\nu_{t-1} = \dot\mu(X_t\tran\theta_\star)\,, \qquad t\in \N^+\,.
	\end{align}
    First, we verify the conditions of this theorem.
    Choose any filtration $\F'$ that makes $(X_t)_{t\in \N^+}$ $\F'$-predictable and
    $(Y_t)_{t\ge 1}$ $\F'$-adapted. Then, $(\nu_t)_{t\in \N}$ is also $\F'$-adapted.

    By \cref{ass:main}\eqref{ass:armset}, $X_t\in B^d_2$ for $t\ge 1$.
    Our definition of $\epsilon_i$ and $S_t$ matches the definition in \cref{thm:novel-concentration}.
    In particular,
    $\epsilon_t = Y_t - \EE{Y_t|\F'_{t-1}} = Y_t - \mu(X_t^\top \theta_\star)$.
    Then,  \cref{eq:lvbound} holds because $\log(2) \le 1$, \cref{lem:mgf-bound} and the choice of $(\nu_t)_{t\in \N}$ in \cref{eq:varid}.
	Note that \cref{lem:mgf-bound} is applicable
    because \cref{ass:main}\eqref{ass:nef} and \eqref{ass:sc} hold.

    Now, note that \cref{eq:varid}, together with the definition of $H_t$ (see \cref{eq:htdef}) implies that
    \begin{align*}
	\widetilde H_t = H_t\,.
	\end{align*}

	Applying \cref{thm:novel-concentration}, we conclude that, with probability at least $1-\delta$,
        \begin{equation*}
            \norm[H_t^{-1}]{S_t}
            \leq \frac{\sqrt{\lambda}}{2M} + \frac{2M}{\sqrt{\lambda}} \left(d + \frac{1}{2}\log\left(\det H_t/\lambda^{d}\right) + \log(1/\delta)\right),
        \end{equation*}
        where we upper-bounded the $\log(2)$ featuring in \cref{thm:novel-concentration} by $1$ for convenience.

        We prepare for upper bounding the determinant $\det H_t$ that appeared on the right-hand side of the last display.
         Applying the usual trace-determinant inequality, say, Note 1, Section 20.2 of \textcite{lattimore2018bandit},
         with the bound $\dot\mu(X_i\tran\theta_\star) \leq L$ and $\norm{X_i} \leq 1$
         which hold due to \cref{ass:main}\eqref{ass:sg} and \cref{ass:main}\eqref{ass:armset}, respectively,
         we have that
        \begin{equation*}
            \det H_t / \lambda^{d} \leq \left(1 + \frac{tL}{\lambda d}\right)^{d},
        \end{equation*}
        which we substitute into the above bound on $\norm[H_t^{-1}]{S_t}$ to complete the proof.
\end{proof}

\clearpage
\section{FORMAL STATEMENT OF REGRET BOUND FOR EVILL}\label{apx:regret-formal}

In the reminder of these appendices, we will show that under the above assumptions, the following formal version of the regret bound given in \cref{thm:main} holds. In the theorem below, and in what follows, we use $\vee$ to denote the binary operator on reals that return the maximum of its arguments, and $\wedge$ the minimum.

\newcommand{\lnd}{\lambda_n\xspace}
\newcommand{\gnd}{\gamma_n\xspace}
\newcommand{\gtd}{\gamma_t\xspace}

\begin{theorem}\label{thm:regret-bound}
    Fix $\delta > 0$ and $n \in \N^+$ and let $\delta' = (\delta/n) \wedge (1/200)$. Consider \cref{alg:EVILL} under \cref{ass:main}, with parameters $\lambda =\lnd$ and $a =\gnd$ given by
    \begin{equation*}
       \lnd = 1\vee \frac{2dM}{S} \log\left(e\sqrt{1+nL/d}\vee 1/\delta \right)
    \end{equation*}
    and
    \begin{equation*}
       \gnd = \sqrt{\lnd}\left(\frac{1}{2M} + S\right) + \frac{2dM}{\sqrt{\lnd}} \log\left(e\sqrt{1+nL/d}\vee 1/\delta \right).
    \end{equation*}
    Suppose further that the prior observations $(X_1, Y_1), \dotsc, (X_\tau, Y_\tau)$ satisfy
    \begin{equation*}
        \max_{x \in \cX} \norm[V_\tau^{-1}]{x} \leq b \spaced{for} b =  \left[22M(1 + M C^2_d(\delta')\widehat D_\star) C^2_d(\delta')\gnd \cdot \sqrt{\kappa}\right]^{-1}\,,
    \end{equation*}
    where $V_\tau = \sum_{i=1}^\tau X_i X_i\tran + \lambda I$ and
    \begin{equation*}
        \kappa = 1 \vee \max_{|u|\leq S} 1/\dot\mu(u), \quad C_d(\delta') = \sqrt{d} + \sqrt{2\log1/\delta'}, \quad \widehat D_\star = \Xi + M\Xi^2 \spaced{where} \Xi = \sqrt{2}\left(\frac{1}{2M} + 2S\right)\,,
    \end{equation*}
    where $\tau$ is a stopping time with respect to the filtration induced by $(X_t,Y_t)_{t\in [n]}$.
    Then, the regret of \cref{alg:EVILL} incurred over the $n$ rounds of interaction, denoted $R(n)$, is upper bounded with probability $1-3\delta$ as
    \begin{align*}
        R(n) \preccurlyeq C_d(\delta') \gnd \sqrt{n \dot\mu_\star dL \log(1 + nL/d \vee 1/\delta)} + C_d^2(\delta') \gnd^2 M dL \log(1+nL/d) + \tau\Delta\,,
    \end{align*}
    where $\preccurlyeq$ hides absolute constants and $\Delta$ is the maximum per-step regret.
\end{theorem}

\begin{remark}\label{remark:warm-up}
    By our discussion in \cref{apx:warm-up}, the $\tau$ featuring in the regret bound of \cref{thm:regret-bound} needs only be as large as order $d/b^2$, which itself is $O(\kappa\cdot(d^{7/2} + d^{3/2} (\log n)^2)\cdot M (1 \vee S^2M^2))$. This is the only term of our regret bound that features $\kappa$. Observe also that, since NEF distributions are non-degenerate by definition, and $\kappa$ depends on the minimum variance of such a distribution over a closed subset of the natural parameter space, $\kappa$ is finite.
\end{remark}

\newcommand{\cEd}{\cE\xspace} %{\cE(\delta)}
\newcommand{\cGt}{\cG_t\xspace} %{\cG_t(\delta')}
\paragraph{Layout of proof} Throughout our proofs, we work in the setting of \cref{thm:regret-bound}, but we may highlight specific parameter choices when these matter. Our proof is laid out as follows:
\begin{itemize}
    \item In \cref{apx:good-event} we introduce the good event $\cEd$, a specialisation of the event where the confidence sets of \cref{apx:confidence-sets} hold, and establish properties that hold on it.
    \item In \cref{apx:other-good-event} we examine the distribution of $\theta_t$, introducing along the way a second set of good events, $\cGt$. We establish useful properties that hold on $\cGt$.
    \item In \cref{apx:p-optimism} we use the properties that hold on $\cEd$ and $\cGt$ to establish that the parameters $\theta_t$ are, in an appropriate sense, optimistic for the true parameters $\theta_\star$ with constant probability.
    \item In \cref{apx:proof-regret}, we use the constant probability optimism together with techniques of \citet{faury2022jointly}, extended to the GLB setting, to prove \cref{thm:regret-bound}.
\end{itemize}
We will rely on the notations and definitions of \cref{apx:defs}. We would also encourage the reader to consult \cref{apx:warm-up} for some intuition on self-concordance and the effect of the prior observations.

\paragraph{Filtration and probability} In the remainder, we will need the filtration $\F = (\F_0, \F_1, \dotsc)$, where
\begin{align*}
\F_t = \sigma( X_1, Y_1, Z_1,Z_1',\dotsc, X_t, Y_t,Z_t,Z_t'),
\end{align*}
the smallest $\sigma$-algebra that makes $X_1, Y_1, Z_1,Z_1',\dotsc, X_t, Y_t,Z_t,Z_t'$ measurable.
Here, and later in the proof, we introduce $Z_t,Z_t'$ for all $t\in [n]$, including $t\le \tau$.
This makes writing the proof easier.
In fact, we will think of
running the algorithm for all $t\in [n]$, including $t\le \tau$, except that when $t\le \tau$, the algorithm `chooses' $X_t$ as given to it in the prior data.
This allows us to define $\hat\theta_0$, $\hat\theta_t$, $W_t$ and $\theta_t$ for all $t\in [n]$.
Note also that since $\{\tau\le t\}\in \sigma(X_1,Y_1,\dots,X_{t},Y_{t})$ by assumption, $\{\tau\le t\} \in \F_{t}$---in other words, $\tau$ is an $\F$-stopping time.

We will write $\E_t$ for the $\F_t$-conditional expectation, and $\P_t$ for the $\F_t$-conditional probability. Inequalities of random variables will hold almost surely.
Throughout, we index $\F_t$-measurable quantities by $t$.

\paragraph{A couple useful claims} The following claims will be used throughout the proofs. Both follow by rote algebra.
\begin{claim}\label{claim:Xi}
    We have $\gnd/\sqrt{\lnd} \leq \Xi$.
\end{claim}
\begin{claim}\label{claim:quadratic}
    For $x \in \R$ and $b,c \geq 0$, $x^2 \leq bx + c \implies x \leq b + \sqrt{c}$.
\end{claim}

\section{THE GOOD EVENT $\cEd$}\label{apx:good-event}

Define the good event
\begin{equation}\label{eq:cedef}
    \cEd = \cap_{t=0}^n \cEd_t\,,
\end{equation}
where
\begin{align*}
\cEd_t = \left\{ \norm[\Ht^{-1}]{g_t(\hat\theta_t) - g_t(\theta_\star)} \leq\gnd\right\}\,.
\end{align*}
Note that $\cEd_t\in \F_t$ holds for all $t\in [n]$.
From \cref{lem:confidence-sets} and observing that for all $\delta > 0$ and $t \in [n]$, $\gnd \geq \gamma_t(\delta,\lnd)$, the following is immediate.
\begin{claim}\label{claim:p-good}
    We have $\P(\cEd) \geq 1-\delta$.
\end{claim}

The next claim will require a little more work.

\begin{claim}\label{claim:D-star-bound}
For any $t\in [n]$, on $\cEd_t$, $D(\hat\theta_t- \theta_\star) \leq \norm{\hat\theta_t-\theta_\star} \leq \widehat D_\star$
and consequently the same holds also on $\cE$.
\end{claim}

\begin{proof}
    The first inequality follows from that the pseudonorm $D$ is bounded by the 2-norm (see \cref{eq:dbound}).
    Turning to the second equality, for any $t \in \N$, let $G_t =G(\hat\theta_t, \theta_\star; \cD_t)$. Then,
    \begin{align*}
    \MoveEqLeft
        \lambda_n \norm{\hat\theta_t-\theta_\star}^2
        = \lambda_n \norm{G_t^{-1} (g_t(\hat\theta_t)-g_t(\theta_\star))}^2
        		 \tag{\cref{claim:g-t-mvt} and $G_t \succeq \lambda_n I$ invertible} \\
        &\leq \norm[G_t^{-1}]{g_t(\hat\theta_t)-g_t(\theta_\star)}^2
        		\tag{$G_t \succeq \lambda_n I$} \\
        &\leq (1+M D(\hat\theta_t- \theta_\star)) \norm[\Ht^{-1}]{g_t(\hat\theta_t)-g_t(\theta_\star)}^2
		        \tag{\cref{claim:GH-bound} and $\Ht = H(\theta_\star;\cD_t)$} \\
        &\leq (1 + M \norm{\hat\theta_t - \theta_\star}) \norm[\Ht^{-1}]{g_t(\hat\theta_t)-g_t(\theta_\star)}^2 \,.
        	\tag{by $D(v)\le \norm{v}$}\\
		&\le (1 + M \norm{\hat\theta_t - \theta_\star})  \gamma_n^2\,. \tag{holds on $\cEd_t$ by the definition of $\cEd_t$}\\
		&= \gamma_n^2 + M \gamma_n^2 \norm{\hat\theta_t - \theta_\star}\,.
    \end{align*}
	Using \cref{claim:quadratic} with $\smash{x = \norm{\hat\theta_t - \theta_\star}}$, we obtain that for all $t \in [n]$,
    \begin{equation*}
        \|\hat\theta_t - \theta_\star\| \leq  \sqrt{\frac{\gnd^2}{\lambda_n}} + M\frac{\gnd^2}{\lambda_n},
    \end{equation*}
    which is bounded by $\widehat D_\star = \Xi + M \Xi^2$ by \cref{claim:Xi}.
\end{proof}

The above together with our choice of $b$, \cref{claim:precondition} and \cref{claim:dot-mu-exp} yield the following claim, which we use repeatedly.

\begin{claim}\label{claim:b-D-star}\label{lem:event-E1}
Fix any $t\in [n]$.
    With the choice of $b$ prescribed by \cref{thm:regret-bound}, on $\cEd_t\cap \{t\ge \tau\}$,
    \begin{equation*}
        MD(\hat\theta_t- \theta_\star) \leq \frac{1}{20C^2_d(\delta')} \leq \frac{1}{20} \spaced{and} \norm[\Ht]{\theta_\star - \hat \theta_t} \leq \frac{21}{20}\gnd.
    \end{equation*}
    Consequently, for any $x \in \cX$,
     on $\cEd_t\cap \{t\ge \tau\}$, $\dot\mu(x\tran\hat\theta_t)$ and $\dot\mu(x\tran\theta_\star)$ are within a factor of $e^{1/20} \leq 11/10$ of one another.
\end{claim}
Note, we will sometimes use $e$ in place of $e^{1/20}$ or $11/10$ in the above, and $2$ in place of $21/20$, so as not to introduce too many unsightly fractions.

\begin{proof}
    For the first inequality, note that by Cauchy-Schwarz and \cref{claim:precondition},
    which can be applied on  $\cEd_t\cap \{t\ge \tau\}\subset \{t\ge \tau\}$, we have
    \begin{equation*}
        MD(\hat\theta_t - \theta_\star) \leq M\norm[H_t^{-1}]{x}\norm[H_{t}]{\hat\theta_t - \theta_\star} \leq b\sqrt{\kappa}M\norm[H_{t}]{\hat\theta_t - \theta_\star}
    \end{equation*}
    and, letting $G_t = G(\hat\theta_t, \theta_\star; \cD_t)$, on $\cEd_t$,
    \begin{align*}
        \norm[\Ht]{\theta_\star - \hat \theta_t}
        &= \norm[\Ht]{G^{-1}_t(g_t(\theta_\star) - g_t(\hat\theta_t))}  \tag{\cref{claim:g-t-mvt}} \\
        &\leq (1+M D(\hat\theta_t- \theta_\star)) \norm[\Ht^{-1}]{g_t(\theta_\star) - g_t(\hat\theta_t)}
        \tag{by \cref{claim:b-D-star}, $H_t\preceq (1+MD(\hat\theta_t- \theta_\star)) G_t$}  \\
        &\leq (1 + M D(\hat\theta_t- \theta_\star))\gnd \tag{Definition on $\cEd_t$} \\
        &\leq (1+M\widehat D_\star) \gamma_n. \tag{\cref{claim:D-star-bound}}
    \end{align*}
    Chaining the inequalities obtained, together with the definition of $b$, and $C_d(\delta')\ge 1$ gives the first inequalities.
    For the second inequality, apply the first inequality at the penultimate line of the above display.

    For the final conclusion, note that by \cref{claim:dot-mu-exp} and the first part of the result,
    \begin{align*}
	\max\left(
	\frac{\dot\mu(x\tran\hat\theta_t)}{\dot\mu(x\tran\theta_\star)},
	\frac{\dot\mu(x\tran\theta_\star)}{\dot\mu(x\tran\hat\theta_t)},
	\right)
	\le \exp(M D( \theta_\star - \hat \theta_t) ) \le \exp(1/20)\,. \tag*{\qedhere}
	\end{align*}
\end{proof}

\section{THE OTHER GOOD EVENTS, $\cGt$}\label{apx:other-good-event}

The next lemma follows from considering the first order optimality conditions at $\theta_t$ and $\hat\theta_{t-1}$ and using \cref{claim:g-t-mvt}.

\begin{lemma}\label{lem:parameter-perturbation}
For any $t\in [n]$,
\begin{align}
-W_t =  g_{t-1}(\theta_t) -g_{t-1}(\hat\theta_{t-1})  = G(\theta_t, \hat\theta_{t-1}; \cD_{t-1})(\theta_t - \hat\theta_{t-1})\,.
\label{eq:master}
\end{align}
Furthermore, with $a =\gnd$, EVILL induces parameter perturbations of the form
    \begin{equation}
        G(\theta_t, \hat\theta_{t-1}; \cD_{t-1})(\theta_t - \hat\theta_{t-1})  =\gnd H(\hat\theta_{t-1}; \cD_{t-1})^{1/2} A_t
        \label{eq:m2}
	\end{equation}
		where $A_t \in \R^d$ is such that its distribution, given
		$\F_{t-1}$, is $\cN(0,I)$.
\end{lemma}
\begin{proof}
We start by establishing \cref{eq:master}.
From the first order optimality conditions and the definition of $g_{t-1}$ (copying partly from \cref{eq:basicopt}),
\[
g_{t-1}(\hat\theta_{t-1}) = \sum_{i=1}^t X_i Y_i = g_{t-1}(\theta_t)+W_t\,,
\]
where
\[
W_{t} = a \lambda^{1/2} Z_t +   a \sum_{i=1}^{t-1} \dot\mu(X_i^\top \hat\theta_{t-1})^{1/2} Z_{t,i}' X_i
\]
and
 $Z_t \sim \cN(0,I_d)$, $Z_t'\sim \cN(0,I_{t-1})$.
Hence,
\begin{align*}
-W_t =  g_{t-1}(\theta_t) -g_{t-1}(\hat\theta_{t-1})  = G(\theta_t, \hat\theta_{t-1}; \cD_{t-1})(\theta_t - \hat\theta_{t-1})\,,
\end{align*}
 where the
last equality is from \cref{claim:g-t-mvt}. This proves \cref{eq:master}.

 Now, notice that,
 given $\F_{t-1}$,
  $W_t$ is zero mean Gaussian with covariance $a^2 H(\hat\theta_{t-1}; \cD_{t-1})$.
It follows that,
 given $\F_{t-1}$,
 $A_t = \frac{1}{a} H(\hat\theta_{t-1}; \cD_{t-1})^{-1/2} (-W_t)$ is zero mean Gaussian with identity covariance.
\end{proof}

We define the event $\cGt$ to be that on which the norm of $A_t$ is not too extreme. In particular, recalling that $C_d(\delta') = \sqrt{d} + \sqrt{2\log 1/\delta'}$, we set
\begin{equation*}
    \cGt = \{ \|A_t\| \leq C_d(\delta') \}.
\end{equation*}

From the standard Gaussian concentration bound (e.g., Theorem II.6 of \textcite{davidson2001local}), we have the following bound for the probability of $\cGt$:

\begin{claim}\label{claim:other-event-prob}
    We have $\P_{t-1}(\cGt) \geq 1-\delta'$.
\end{claim}

Mirroring \cref{apx:good-event}, we now upper bound $D(\theta_t- \hat\theta_{t-1})$ on $\cGt$.

\begin{claim}\label{claim:D-bound}
    With the choice $a =\gnd$, $\lambda=\lambda_n$, on the event $\cGt\cap \{ t>\tau\}$,
    \begin{equation*}
        D(\theta_t- \hat\theta_{t-1}) \leq \norm{\theta_t -\hat\theta_{t-1}} \leq  C_d^2(\delta') \widehat D_\star.
    \end{equation*}
\end{claim}

\begin{proof}
Again, the first inequality follows from $D(v)\le \norm{v}$ with $v = \theta_t- \hat\theta_{t-1}$ (see \cref{eq:dbound}).

For the second inequality,  introduce the shorthand $\hH_{t-1} =H(\hat\theta_{t-1}; \cD_{t-1})$ and
$G_t = G(\theta_t, \hat\theta_{t-1}; \cD_{t-1})$. Let also $S_2^{d-1}$ be the unit sphere in $\R^d$: $S_2^{d-1} = \{ x\in \R^d\,:\, \norm{x}=1\}$. Then, we have the following:
    \begin{align*}
         \lambda_n \norm{\theta_t -\hat\theta_{t-1}}^2
        &=  \lambda_n \gnd^2\norm{G_t^{-1} \hH_{t-1}^{1/2} A_t}^2
        \tag{\cref{eq:m2} of \cref{lem:parameter-perturbation}}
        \\
        &=  \gnd^2\norm{G_t^{-1/2} \hH_{t-1}^{1/2} A_t}^2
        \tag{$G_t=G(\theta_t, \hat\theta_{t-1}; \cD_{t-1}) \succeq \lambda_n  I$}
        \\
        & \leq  \gnd^2 C_d^2(\delta') \norm{G_t^{-1/2} \hH_{t-1}^{1/2}}^2
        \tag{definition of $\cGt$}\\
        & = \gnd^2 C_d^2(\delta') \,\sup_{x\in S_2^{d-1}} x^\top \hH_{t-1}^{1/2} G_t^{-1} \hH_{t-1}^{1/2} x \\
        & \le \gnd^2 C_d^2(\delta') \, (1+MD(\theta_t- \hat\theta_{t-1}))\tag{\cref{claim:GH-bound}}\\
        &\leq \gnd^2 C_d^2(\delta')  \,(1 + M\norm{\theta_t -\hat\theta_{t-1}}). \tag{$D(v)\le \norm{v}$}\\
        &= \gnd^2 C_d^2(\delta') + M \gnd^2 C_d^2(\delta') \norm{\theta_t -\hat\theta_{t-1}}.
    \end{align*}
    Now, by \cref{claim:quadratic},
    \begin{equation*}
        \norm{\theta_t -\hat\theta_{t-1}}
        \leq C_d(\delta')\sqrt{\frac{\gnd^2}{\lambda_n}} + M C_d^2(\delta')  \frac{\gnd^2}{\lambda_n}.
    \end{equation*}
    Since $C_d(\delta') \geq \sqrt{d} \geq 1$,
    by \cref{claim:Xi}, this is upper bounded by $C^2_d(\delta') \widehat D_\star$.
\end{proof}

\begin{claim}\label{claim:b-D-t}\label{lem:event-E2}
    With the choice of $b$ prescribed by \cref{thm:regret-bound}, for all $t \in [n]$, on $\cE_{t-1} \cap \cGt\cap \{t>\tau\}$,
    \begin{equation*}
        M D(\theta_t- \hat\theta_{t-1}) \leq \frac{1}{20} \spaced{and} \norm[H(\hat\theta_{t-1} ; \cD_{t-1})]{\theta_t - \hat \theta_{t-1}} \leq \frac{21}{20} C_d(\delta')\gnd.
    \end{equation*}
    Consequently, for any $x \in \cX$, $\dot\mu(x\tran\theta_t)$ and $\dot\mu(x\tran\hat\theta_{t-1})$ are within a factor of $e^{1/20} \leq 11/10$ of one another.
\end{claim}
Again, we will sometimes use $e$ in place of $e^{1/20}$ or $11/10$ in the above, and $2$ in place of $21/20$.
\begin{proof}
    Again write $\hH_{t-1} =H(\hat\theta_{t-1}; \cD_{t-1})$ and
    $G_t = G(\theta_t, \hat\theta_{t-1}; \cD_{t-1})$.
    For the first inequality, letting $X = \argmax_{y \in \cX} |y\tran (\theta_t- \hat\theta_{t-1})|$,
    \begin{align*}
        M D(\theta_t- \hat\theta_{t-1})
        &\leq M\norm[\hH_{t-1}^{-1}]{X}\norm[\hH_{t-1}]{\theta_t- \hat\theta_{t-1}} \tag{Cauchy-Schwarz}\\
        &\leq \frac{11}{10}M\norm[H_{t-1}^{-1}]{X}\norm[\hH_{t-1}]{\theta_t- \hat\theta_{t-1}}
        \tag{by \cref{lem:event-E1}, $H_{t-1}\preceq \frac{11}{10}\hH_{t-1}$ on $\cE_{t-1} \cap \{ t > \tau\}$}\\
        &\leq \frac{11}{10}b\sqrt{\kappa}M\norm[\hH_{t-1}]{\theta_t- \hat\theta_{t-1}} \tag{on $\{t > \tau \}$, by \cref{claim:precondition}}
    \end{align*}
    and
    \begin{align*}
        \norm[\hH_{t-1}]{\theta_t - \hat \theta_{t-1}}
        &=\gnd\norm{\hH_{t-1}^{1/2}  G_t^{-1} \hH_{t-1}^{1/2} A_t } \\
        & \leq\gnd C_d(\delta') \norm{\hH_{t-1}^{1/2}  G_t^{-1} \hH_{t-1}^{1/2} } \tag{on $\cG_t$} \\
        &\leq\gnd C_d(\delta') (1 + M D(\theta_t- \hat\theta_{t-1})) \tag{\cref{claim:GH-bound}}\\
        &\leq\gnd C_d(\delta') (1 + C^2_d(\delta') M \widehat D_\star). \tag{\cref{claim:D-bound}}
    \end{align*}
    The first inequality then follows from chaining the last two displays and using the definition of $b$.
    For the second inequality, apply the first inequality at the penultimate line of the above display.
    The final part follows with the same proof as the analogue statement in \cref{lem:event-E1}.
\end{proof}

\section{LOWER BOUNDING THE PROBABILITY OF OPTIMISM}\label{apx:p-optimism}

The aim of this appendix is to show that each $\theta_t$ is optimistic with at least a constant probability, where optimistic means that it will belong to the set of optimistic parameters, defined as follows.
\begin{definition}[Optimistic parameters]\label{def:ThetaOpt} We let $\ThetaOpt$ denote the set of parameters optimistic for $\theta_\star$, defined as
    \begin{equation*}
        \ThetaOpt = \{ \theta \in \R^d \colon \max_{x \in \cX} x\tran\theta \geq x_\star\tran\theta_\star\}.
    \end{equation*}
\end{definition}
This proof of constant optimism is where \citet{kveton2019randomized} needed to introduce the orthogonality assumption, that we do away with by employing tools based on self-concordance from \citet{sun2019generalized,faury2022jointly}. This appendix is the main technical contribution in the context of proving \cref{thm:regret-bound}.

In this section, we will use the following shorthands:
\begin{equation*}
    G_t = G(\theta_t, \hat\theta_{t-1}; \cD_t), \qquad \hH_{t-1} = H(\hat\theta_{t-1}; \cD_{t-1}), \qquad H_{t-1} = H(\theta_\star, \cD_{t-1})
\end{equation*}
(the shorthands $\hH_{t-1}$ and $H_{t-1}$ were introduced beforehand.)
We will also work with the quantity
\[
E_t = G_t - \hH_{t-1}\,.
\]
We start with a result concerning $E_t$:

\begin{claim}\label{claim:opt-local-bound}
If $M D(\theta_t- \hat\theta_{t-1}) \leq 1$, then $\norm{\hH_{t-1}^{-1/2}E_t \hH_{t-1}^{-1/2}} \leq M D(\theta_t- \hat\theta_{t-1})$.
\end{claim}

\begin{proof}
    Choose $x \in \R^d$ with $\|x\|=1$ such that $\norm{\hH_{t-1}^{-1/2}E_t \hH_{t-1}^{-1/2}} = x\tran \hH_{t-1}^{-1/2} E_t \hH_{t-1}^{-1/2} x$, and write $y=\hH_{t-1}^{-1/2} x$. By \cref{lem:alpha-bound},
    \begin{align*}
        y\tran E_t y
        &= y\tran \sum_{i=1}^t \left(\alpha(X_i^\top\theta_t,X_i^\top\hat\theta_{t-1}) - \dot\mu(X_i^\top \hat\theta_{t-1})\right) X_i X_i\tran y \\
        &\leq \left(
        h(M D(\theta_t- \hat\theta_{t-1}))-1
        \right) \cdot y\tran \sum_{i=1}^t  \dot\mu(X_i^\top \hat\theta_{t-1}) X_i X_i\tran y.
    \end{align*}
    Now, we note that for $z \leq 1$, $h(z)-1=\frac{e^z-1}{z} -1 \leq z$, and recall that we assumed $MD(\theta_t- \hat\theta_{t-1}) \leq 1$. Thus,
    \begin{equation*}
        \norm{\hH_{t-1}^{-1/2}E_t \hH_{t-1}^{-1/2}}
        \leq MD(\theta_t- \hat\theta_{t-1}) \cdot x\tran \hH_{t-1}^{-1/2} (\hH_{t-1} - \lambda I) \hH_{t-1}^{-1/2} x
        \leq MD(\theta_t- \hat\theta_{t-1}).\qedhere
    \end{equation*}
\end{proof}

\begin{lemma}\label{lem:event-E3}
	Fix $t\in [n]$.
    On the event $\cEd\cap \{t> \tau\}$,
    \begin{equation*}
        \P_{t-1}\left(\{\theta_t \in \ThetaOpt \} \cap \cGt \right) \geq 1/200.
    \end{equation*}
\end{lemma}

\begin{proof}
    Consider an arbitrary $t \in[n]$. Write $\tP(\, \cdot\,) = \P_{t-1}(\,\, \cdot\, \cap \cGt)$.
    Let $\tilde\cE_{t-1} = \{ \snorm{ \theta_\star-\hat\theta_{t-1} }_{H_{t-1}} \le 2\gamma_n \}$.
    By \cref{lem:event-E1}, $\cEd_{t-1}\cap\{ t-1\ge \tau\}=\cEd_{t-1}\cap\{ t> \tau\} \subset \tilde\cE_{t-1}$.
    Hence, on $\cE\cap \{ t> \tau\}$,
    \begin{align*}
        \MoveEqLeft
        p_{t-1}
        \doteq \tP\left(\theta_t \in \ThetaOpt\right) \\
        & = \tP\left(X_t^\top\theta_t \geq x_\star^\top\theta_\star\right) \\
        &\ge \tP\left(x_\star^\top(\theta_t - \hat\theta_{t-1}) \geq x_\star\tran(\theta_\star - \hat{\theta}_{t-1}) \right) \tag{definition of $X_t$}\\
        &\ge \tP\left(x_\star^\top(\theta_t - \hat\theta_{t-1}) \geq x_\star\tran(\theta_\star - \hat{\theta}_{t-1})\,, \tilde\cE_{t-1} \right) \\
        &\geq \tP\left(x_\star^\top(\theta_t-\hat\theta_{t-1}) \geq 2\gnd \norm[H_{t-1}^{-1}]{x_\star}\,, \tilde\cE_{t-1}\right),
        \tag{Cauchy-Schwarz and the definition of $\tilde\cE_{t-1}$}\\
        &= \mathbf{1}_{\tilde\cE_{t-1}} \tP\left(x_\star^\top(\theta_t-\hat\theta_{t-1}) \geq 2\gnd \norm[H_{t-1}^{-1}]{x_\star}\right),
        \tag{$\tilde\cE_{t-1}\in \F_{t-1}$}\\
        &=\tP\left(x_\star^\top(\theta_t-\hat\theta_{t-1}) \geq 2\gnd \norm[H_{t-1}^{-1}]{x_\star}\right)\,.
        \tag{we are on $\cE\cap\{t> \tau\}$ and $\cE \cap\{t> \tau\}\subset \tilde\cE_{t-1}$}
    \end{align*}
    By \cref{lem:parameter-perturbation}, $x_\star\tran(\theta_t - \hat\theta_{t-1}) = x_\star\tran G_t^{-1} \hH_{t-1}^{1/2} A_t$,
    where, as before, $G_t = G(\theta_t,\hat\theta_{t-1};\cD_{t-1})$.
    By the
    Woodbury matrix identity, letting $E_t = G_t - \hH_{t-1}$,
    \begin{equation*}
        G^{-1}_t
        = (\hH_{t-1} + E_t)^{-1}
        = \hH_{t-1}^{-1} - \hH_{t-1}^{-1} E_t (\hH_{t-1} + E_t)^{-1}
        = \hH_{t-1}^{-1} - \hH_{t-1}^{-1} E_t G_t^{-1},
    \end{equation*}
    and so we have that
    \begin{align}
        p_{t-1}
        &\geq \tP\left(x_\star^\top \hH_{t-1}^{-1/2}A_t\geq x_\star^\top \hH_{t-1}^{-1}E_tG_t^{-1} \hH_{t-1}^{1/2} A_t +2\gnd\norm[H_{t-1}]{x_\star}\right) \nonumber \\
        &\geq \tP\left(x_\star^\top \hH_{t-1}^{-1/2}A_t\geq x_\star^\top \hH_{t-1}^{-1}E_tG_t^{-1} \hH_{t-1}^{1/2} A_t +2\gnd\norm[H_{t-1}]{x_\star},\cE_{t-1}\cap \{t>\tau\}\right)\,.        \label{eq:e3-bound-p}
    \end{align}
    Using Cauchy-Schwarz, and then \cref{lem:parameter-perturbation} and \cref{lem:event-E2}, on $\cGt$,
    \begin{align*}
        x_\star^\top \hH_{t-1}^{-1}E_tG_t^{-1} \hH_{t-1}^{1/2} A_t
        &\leq \norm[\hH_{t-1}^{-1}]{x_\star}\norm{\hH_{t-1}^{-1/2}E_t \hH_{t-1}^{-1/2}} \norm[\hH_{t-1}]{G_t^{-1} \hH_{t-1}^{1/2} A_t} \\
        &= \norm[\hH_{t-1}^{-1}]{x_\star}\norm{\hH_{t-1}^{-1/2}E_t \hH_{t-1}^{-1/2}} \norm[\hH_{t-1}]{\theta_t - \hat\theta_{t-1}}  \\
        &\leq 2\gnd C_d(\delta') \norm[\hH_{t-1}^{-1}]{x_\star}\norm{\hH_{t-1}^{-1/2}E_t \hH_{t-1}^{-1/2}}\,.
    \end{align*}
    Now, observe that on $\cGt \cap \cE_{t-1}\cap \{t>\tau\}$, by \cref{claim:b-D-star} and \cref{claim:b-D-t}, $MD(\theta_t- \hat\theta_{t-1}) \leq \frac{1}{20 C_d(\delta')} \leq \frac{1}{20}$ and $MD(\hat\theta_{t-1}- \theta_\star) \leq \frac{1}{20}$.
    The first of these lets us apply \cref{claim:opt-local-bound} to obtain that
    \begin{equation*}
        \norm{\hH_{t-1}^{-1/2}E_t \hH_{t-1}^{-1/2}} \leq \frac{1}{20C_d(\delta')}, \spaced{which gives the bound} x_\star^\top \hH_{t-1}^{-1}E_tG_t^{-1} \hH_{t-1}^{1/2} A_t \leq \frac{1}{10}\gnd \norm[\hH_{t-1}^{-1}]{x_\star}.
    \end{equation*}
    Combining the first with the second, $MD(\theta_t- \theta_\star) \leq MD(\theta_t- \hat\theta_{t-1}) + MD(\hat\theta_{t-1}- \theta_\star) \leq \frac{1}{10}$, and so by \cref{claim:dot-mu-exp}, $\norm[H_{t-1}^{-1}]{x_\star} \leq \sqrt{e^{\frac{1}{10}}}\norm[\hH_{t-1}^{-1}]{x_\star} \leq \frac{11}{10}\norm[\hH_{t-1}^{-1}]{x_\star}$ (note the norm changed from $H_{t-1}^{-1}$-weighted to $\hH_{t-1}^{-1}$-weighted). Using these two bounds in \cref{eq:e3-bound-p}, on $\cE \cap \{t>\tau\}$, we have that
    \begin{align*}
        \MoveEqLeft p_{t-1}
        \geq \tP\left(x_\star^\top \hH_{t-1}^{-1/2}A_t \geq \frac{23}{10}\gnd \norm[\hH_t^{-1}]{x_\star}, \cE_{t-1}\cap \{t>\tau\} \right)
        	\\
        &= \mathbf{1}_{\cE_{t-1}\cap \{t>\tau\}}
        	\tP\left(x_\star^\top \hH_{t-1}^{-1/2}A_t \geq \frac{23}{10}\gnd \norm[\hH_t^{-1}]{x_\star} \right)
				\tag{$\cE_{t-1}, \{t>\tau\}\in \F_{t-1}$, as $\tau$ is an $\F$-stopping time}\\
        &= \tP\left(x_\star^\top \hH_{t-1}^{-1/2}A_t \geq \frac{23}{10}\gnd \norm[\hH_t^{-1}]{x_\star} \right)
        		\tag{$\cE\cap \{t>\tau\} \subset \cE_{t-1}\cap \{t>\tau\}$}\\
        &\geq \P_{t-1} \left(x_\star^\top \hH_{t-1}^{-1/2}A_t \geq \frac{23}{10}\gnd \norm[\hH_t^{-1}]{x_\star}\right) - (1 - \P_{t-1}(\cGt)) \\
        &\geq \P_{t-1} \left(x_\star^\top \hH_{t-1}^{-1/2}A_t \geq \frac{23}{10}\gnd \norm[\hH_t^{-1}]{x_\star}\right) - \delta',
    \end{align*}
    where the middle inequality is a standard result for lower bounding the probability of an intersection of two events, and the third follows from \cref{claim:other-event-prob}. Now, observe that conditioned on $\F_{t-1}$, $x_\star^\top \hH_{t-1}^{-1/2}A_t$ is a centred Gaussian random variable with standard deviation $\gnd \norm[\hH_t^{-1}]{x_\star}$. Thus, together with our assumption on $\delta'$, looking up the relevant Gaussian tail probability, we conclude that, on $\cEd\cap \{t>\tau\}$,
    \begin{equation*}
        p_{t-1} \geq \P_{t-1} \left(x_\star^\top \hH_{t-1}^{-1/2}A_t
        \geq \frac{23}{10}\gnd \norm[\hH_t^{-1}]{x_\star}\right) - \delta' = \cN(0,1)(\{z \in \R \colon z > 23/10 \}) - \delta' \geq 1/200. \qedhere
    \end{equation*}
\end{proof}

\section{PROOF OF REGRET BOUND FOR EVILL}\label{apx:proof-regret}

Much of the proofs in this final appendix largely follows the regret bound given by \citet{faury2022jointly} for their Thompson-sampling-type algorithm for the logistic bandit setting, but fixes an error present therein.\footnote{We thank Marc Abeille for helping us  fix the error in the \citet{faury2022jointly} manuscript.}

To start our bound, we write $\wR$ for the regret over steps $\tau+1,\dotsc,n$, and break  this up as
\begin{align}\label{eq:regret_decomposition}
    \wR = \sum_{t=\tau+1}^n
    \underbrace{\mu(x_\star^\top \theta_\star) - \mu(X_t^\top\theta_t)}_{\rphe}
    + \underbrace{\mu(X_t^\top\theta_t) - \mu(X_t^\top \theta_\star)}_{\rpred}.
\end{align}
The first term is controlled if the algorithm plays optimistically, while the second is controlled if the reward parameter of the arm chosen is well predicted by the parameter vector that governs the action choice.
Hence, the naming of these two terms.

The following two lemmas, proven in \cref{apx:rphe} and \cref{apx:rpred} respectively, bound these two terms.
\begin{lemma}\label{lem:bound-rphe}
    Let $p=1/200$ and fix $t\in [n]$.
    On the event $\cEd\cap \cG_t\cap \{t>\tau\}$, for $\delta'$ and $b$ as in \cref{thm:regret-bound},
    \begin{align*}
        \rphe \leq 7 C_d(\delta')\gnd \sqrt{\dot\mu(x_\star\tran\theta_\star)} \EC [\sqrt{\dot\mu(X_t\tran\hat\theta_{t-1})}\norm[\hH_{t-1}^{-1}]{X_t}]/p.
    \end{align*}
\end{lemma}

\begin{lemma}\label{lem:bound-rpred}
 Fix $t\in [n]$.
    On event $\cEd\cap \cG_t\cap \{t>\tau\}$, for $b$ as in \cref{thm:regret-bound},
    \begin{equation*}
        \rpred\leq 5(1+C_d(\delta'))\gnd \dot\mu(X_t\tran\hat\theta_{t-1}) \norm[\hH^{-1}_{t-1}]{X_t}.
    \end{equation*}
\end{lemma}

Our proof of \cref{thm:regret-bound} will also use the following simple claim, proven in \cref{apx:dot-mu-sum}, and the version of the elliptical potential lemma stated thereafter, which follows from Lemma 15 of \textcite{faury2020improved} combined with the usual trace-determinant inequality, say, Note 1, Section 20.2 of \textcite{lattimore2018bandit}.
\begin{claim}\label{lem:dot-mu-sum}
    $\sum_{t=\tau+1}^n \dot\mu(X_t^\top \theta_\star) \leq n\dot\mu(x_\star^\top \theta_\star) + M \wR$.
\end{claim}

\begin{lemma}[Elliptical potential lemma]\label{lem:epl}
    Fix $\lambda, A > 0$. Let $\{a_t\}_{t=1}^\infty$ be a sequence in $AB^d_2$ and let $V_0 = \lambda I$. Define $V_{t+1} = V_t + a_{t+1} a_{t+1}\tran$ for each $t \in \N$. Then, for all $n \in \N^+$,
    \begin{equation*}
        \sum_{t=1}^n \norm[V_{t-1}^{-1}]{a_t}^2 \leq 2 d\max\left\{1,\frac{A^2}{\lambda}\right\} \log \left( 1 + \frac{n A^2}{d\lambda}\right) \,.
    \end{equation*}
\end{lemma}

\newcommand{\sdms}{\sqrt{\dot\mu(x_\star\tran\theta_\star)}}

\begin{proof}[Proof of \cref{thm:regret-bound}]
Assume that $\cE$, defined via \cref{eq:cedef}, holds.
We start with the regret decomposition of \cref{eq:regret_decomposition}, use \cref{lem:bound-rphe,lem:bound-rpred} to bound the two terms respectively, and write the result in terms of
\begin{equation*}
    M_t = \EC \left[\sqrt{\dot\mu(X_t\tran\hat\theta_{t-1})}\norm[\hH_{t-1}^{-1}]{X_t}\right] - \sqrt{\dot\mu(X_t\tran\hat\theta_{t-1})}\norm[\hH_{t-1}^{-1}]{X_t}\,.
\end{equation*}
That is, using $\preccurlyeq$ to indicate inequalities up to absolute constants (including $p=1/200$),
\begin{align*}
    \wR
    &\preccurlyeq \gnd C_d(\delta')
     \sum_{t=1}^n \one{t>\tau} \bigg(\underbrace{\dot\mu(X_t\tran \hat\theta_{t-1})\norm[\hH^{-1}_{t-1}]{X_t}}_{A_t} + \sdms \underbrace{ \sqrt{\dot\mu(X_t\tran\hat\theta_{t-1})} \norm[\hH_{t-1}^{-1}]{X_t}}_{B_t}
    + \sdms M_t \bigg) \,.
\end{align*}

Consider first bounding $\sum_{t=1}^n\one{t>\tau} M_t$.
For this bound we drop the assumption that $\cE$ holds: we give an upper bound that holds with probability $1-\delta$ (so for the final bound we will need to use an extra union bound with the error event introduced in this step).
The sum is upper bounded with
the help of the Azuma-Hoeffding's inequality (say, Corollary 2.20 in \textcite{wainwright2019high}).
For this, notice that
$( \one{t>\tau} M_{t})_{t\in [n]}$ is a martingale difference sequence with respect to $\F$ (this follows because $\one{t>\tau} = \one{t-1\ge \tau}\in \F_{t-1}$ as $\tau$ is an $\F$-stopping time)
and each $M_t$ satisfies $|M_t| \leq 2\sqrt{L}$. This latter follows because
$\norm[\hH_{t-1}^{-1}]{X_t}\le \norm{X_t}/\lambda \le \norm{X_t}\le 1$ since by choice $\lambda\ge 1$.
Then, from Azuma-Hoeffding, we conclude that with probability $1-\delta$,
\begin{equation*}
    \sum_{t=1}^n \one{t>\tau} M_t \leq \sqrt{8L n \log 2/\delta}\,.
\end{equation*}

Now consider bounding $\sum_{t=1}^n \one{t>\tau} A_t$. We upper-bound this sum on $\cE$.
By \cref{lem:event-E1}, which can be used since $\cE_t \subset \cE$,
\begin{equation*}
    \dot\mu(X_t\tran \hat\theta_{t-1}) \leq \sqrt{e\dot\mu(X_t\tran \theta_\star)\dot\mu(X_t\tran \hat\theta_{t-1})}.
\end{equation*}
Using this and Cauchy-Schwarz,
\begin{align*}
    \sum_{t=1}^n \one{t>\tau} A_t
    \leq \sqrt{e\sum_{t=\tau+1}^n \dot\mu(X_t\tran \theta_\star)} \sqrt{\sum_{t=\tau+1}^n \dot\mu(X_t\tran \hat\theta_{t-1}) \norm[\hH^{-1}_{t-1}]{X_t}^2}
\end{align*}
Now, by \cref{lem:dot-mu-sum}, $\sum_{t=\tau+1}^n \dot\mu(X_t\tran \theta_\star) \leq n\dot\mu(x_\star^\top \theta_\star) + M \wR$. Furthermore, by \cref{lem:epl} applied to the sequence $a_1, a_2, \dotsc$ given by $a_t = \sqrt{\dot\mu(X_t \tran \hat\theta_{t-1})} X_t$, using
$\snorm{a_t} \leq \sqrt{L}$, we get
    \begin{equation*}
        \sum_{t=\tau+1}^n \dot\mu(X_t\tran \hat\theta_{t-1}) \norm[\hH^{-1}_{t-1}]{X_t}^2 \leq 2d L \log\left(1 + \frac{nL}{d\lnd}\right) =: Q_n.
    \end{equation*}

Now, with an     entirely analogous argument that is omitted as it would just repeat the previous argument,
we get than on $\cE$,
$\sum_{t=1}^n \one{t>\tau} B_t \le Q_n$ also holds.

Putting together these three bounds and using the concavity of the square root function, we have that on $(\cap_{t=1}^n \cGt) \cap \cEd\cap\{t>\tau\}$, with probability $1-\delta$,
\begin{align*}
    \wR \preccurlyeq \gnd C_d(\delta') \left( \sqrt{Q_n M \wR} + \sqrt{n \dot\mu_\star (Q_n + L \log 2/\delta)} \right).
\end{align*}
Now, note the above is of the form $x^2 \leq bx + c$ for $x^2 = \wR$ and $b,c \geq 0$. Applying \cref{claim:quadratic} and squaring, we conclude that on $(\cap_{t=1}^n \cGt) \cap \cEd\cap\{t>\tau\}$, with probability $1-\delta$,
\begin{equation*}
    \wR \preccurlyeq C_d(\delta') \gnd \sqrt{n \dot\mu_\star (Q_n + L \log 2/\delta)} + C_d^2(\delta') \gnd^2 M Q_n\,.
\end{equation*}
We then take $R(n) \leq \wR + \tau \Delta$, where $\Delta$ is the maximal per-step regret, and observe that, by \cref{claim:p-good} and \cref{claim:other-event-prob}, with our choice of $\delta'$, $\P((\cap_{t=1}^n \cGt) \cap \cEd) \geq 1-2\delta$. We then apply a union bound for a $1-3\delta$ probability bound, and simplify somewhat.
\end{proof}

\subsection{Bounding $\rphe$ (proof of \cref{lem:bound-rphe})}\label{apx:rphe}

For this proof, we will work with the function $J(\theta) = \max_{x\in \cX} x\tran \theta$. In the convex analysis literature, the function $J$ is called the \emph{support function} of $\cX$. It is well known that the support function $J$ is convex, and for $\cX \subset \R^d$ compact, any $\theta \in \R^d$, any $x \in \argmax_{y \in \cX} y\tran \theta$ is a subgradient of $J$ at $\theta$, but we include a proof for completeness.
\begin{claim}\label{claim:subgrads}
    Let $J$ be the support function of a compact set $\cX \subset \R^d$.
    Then for any $\theta \in \R^d$, any $x \in \argmax_{y \in \cX} y\tran\theta$ is a subgradient of $J$ at $\theta$.
\end{claim}
\begin{proof}
    Note that for any $\theta' \in \R^d$,
    \begin{equation*}
        J(\theta) + x\tran(\theta'-\theta) = x\tran\theta' \leq \max_{y\in \cX} y\tran\theta' = J(\theta'),
    \end{equation*}
    which is the definition of $x$ being a subgradient of $J$ at $\theta$.
\end{proof}

\begin{proof}[Proof of \cref{lem:bound-rphe}]
    From definition of $\alpha$,
    \begin{equation*}
        \rphe = \mu(x_\star^\top \theta_\star) - \mu(X_t^\top\theta_t) = \alpha(x_\star^\top \theta_\star,X_t^\top\theta_t)(x_\star^\top \theta_\star - X_t^\top\theta_t) =\alpha(J(\theta_\star),J(\theta_t))(J(\theta_\star) - J(\theta_t)).
    \end{equation*}
    Assume from now on that $\cE \cap \cG_t\cap \{t>\tau\}$ holds.

    Using the convexity of $J$, that $X_t$ is a subgradient of $J$ at $\theta_t$, and that $x_\star$ is a subgradient of $J$ at $\theta_\star$,
    \begin{equation*}
        M|J(\theta_\star) - J(\theta_t)|
        \leq M\max\{|x_\star^\top (\theta_\star - \theta_t)|, |X_t^\top(\theta_\star - \theta_t)|\}
        \leq MD(\theta_t- \theta_\star) \leq MD(\theta_t- \hat\theta_{t-1}) + MD(\hat\theta_{t-1}- \theta_\star)\leq 1,
    \end{equation*}
    where the last inequality follows by \cref{claim:b-D-star} and \cref{claim:b-D-t}.
    Now, for $\alpha(J(\theta_\star),J(\theta_t))$, using self-concordance, that is \cref{claim:dot-mu-exp} 	together with the last display, bounding crudely, we have that $\alpha(J(\theta_\star),J(\theta_t))\leq 2 \dot\mu(J(\theta_\star))$, and also that $\alpha(J(\theta_\star),J(\theta_t))\leq 2 \dot\mu(J(\theta_t))$, and so
    \begin{equation}\label{eq:rphe-first-bound}
        \rphe \leq 2\sqrt{\dot\mu(x_\star\tran\theta_\star)} \cdot \sqrt{\dot\mu(X_t\tran\theta_t)} (J(\theta_\star) - J(\theta_t)).
    \end{equation}

    \textbf{The `and so' part above does not follow, since $J(\theta_\star) - J(\theta_t)$ is not known to be nonnegative. This is the error discussed on the first page. While the error looks at first glance like something that may be fixed with a local patch, it is not; see \citet{perneczky2026variance} for a correct argument.}

    We turn to bounding $\sqrt{\dot\mu(X_t\tran\theta_t)} (J(\theta_\star) - J(\theta_t))$.
    Let
    \[
    \ThetaOpt = \{\theta : J(\theta) \ge J(\theta_\star)\} \spaced{and}
     \Theta_{t-1} =\left\{\theta \in \R^d \colon \norm[\hH_{t-1}]{\theta - \hat\theta_{t-1}} \leq \tfrac{21}{20} C_d(\delta')\gnd\right\}\,,
     \]
    and observe that $\ThetaOpt$ matches that defined in \cref{def:ThetaOpt}. Now let $p_{t-1} = \P_{t-1}(\theta_t \in \ThetaOpt \cap \Theta_{t-1})$ and
    \begin{equation*}
        Q(\,\cdot\,) =
        \begin{cases}
            \P_{t-1}(\theta_t \in \cdot \,\cap\, \ThetaOpt \cap \Theta_{t-1}) / p_{t-1}\,, &  p_{t-1} > 0 \,; \\
            \text{an arbitrary probability measure}\,, & \text{otherwise}.
        \end{cases}
    \end{equation*}
	On $\cE\cap \{t>\tau\}$, we claim that
	\begin{align}
	p_{t-1}\ge p=1/200>0 \label{eq:ptlb}
	\end{align}
	hence we are in the first case of the above definition.
	Indeed,
    \begin{align*}
	\MoveEqLeft p_{t-1} = \P_{t-1}(\theta_t \in \ThetaOpt \cap \Theta_{t-1}) \\
	& =  \P_{t-1}(\{ \theta_t \in \ThetaOpt \} \cap  \{\theta_t \in \Theta_{t-1}\})  \\
	& \ge  \P_{t-1}(\{ \theta_t \in \ThetaOpt \} \cap  \cG_t \cap \{ t>\tau \} )
	\tag{because  $\{\theta_t \in \Theta_{t-1}\} \supseteq \cG_t\cap \{t>\tau\}$ by \cref{lem:event-E2}} \\
	& = \one{t>\tau}  \P_{t-1}(\{ \theta_t \in \ThetaOpt \} \cap  \cG_t ) \tag{because $\tau$ is $\F$-adapted}\\
	& =  \P_{t-1}(\{ \theta_t \in \ThetaOpt \} \cap  \cG_t ) \tag{because we are on $\cE \cap \{t>\tau\}$}\\
	& \ge  p=1/200 \,. \tag{by \cref{lem:event-E3}}
	\end{align*}

    Since we are on $\cG_t\cap \{t>\tau\}$, by \cref{lem:event-E2}, $\theta_t \in \Theta_{t-1}$, and so
    \begin{equation*}
        \sqrt{\dot\mu(J(\theta_t))} (J(\theta_\star) - J(\theta_t)) \leq \sqrt{\dot\mu(X_t\tran \theta_{-})} (J(\theta_\star) - J(\theta_{-})) \,,
    \end{equation*}
    where
    we let $\theta_{-}$ be a maximiser of $\sqrt{\dot\mu(X_t\tran \theta)} (J(\theta_\star) - J(\theta))$ over $\theta \in \Theta_{t-1}$.
    Now, observe that
    \begin{align*}
    \MoveEqLeft
        \sqrt{\dot\mu(X_t\tran \theta_-)} (J(\theta_\star) - J(\theta_{-}))
        \leq \int \sqrt{\dot\mu(X_t\tran \theta_-)}\, (J(\theta_+) - J(\theta_{-})) Q(d\theta_+)
        \tag{on $\cE \cap\{t>\tau\}$, by \cref{eq:ptlb}, $\supp Q \subset \ThetaOpt$ and for $\theta_+\in\ThetaOpt$, $J(\theta_+)\ge J(\theta_\star)$, $Q$ is a prob.\ m.}
        \\
        &= \frac{1}{p_{t-1}} \EC[\sqrt{\dot\mu(X_t\tran \theta_-)}\, (J(\theta_t) - J(\theta_{-}))\1{\theta_t \in \ThetaOpt \cap \Theta_{t-1}}] \tag{definition of $Q$}\\
        &\leq \frac{1}{p} \EC[\sqrt{\dot\mu(X_t\tran \theta_-)}\, |J(\theta_t) - J(\theta_{-})|\1{\theta_t \in \Theta_{t-1}}] \,.
        \tag{$p_{t-1}\ge p$, introducing $|\cdot|$ and dropping $\ThetaOpt$}
    \end{align*}
Now, observe that using that $\theta_-, \hat\theta_{t-1} \in \Theta_{t-1}$, the proof of \cref{claim:b-D-t} gives us that $\sqrt{\dot\mu(X_t\tran \theta_-)} \leq \sqrt{e \dot\mu(X_t\tran \hat\theta_{t-1})}$. Also,
    \begin{align*}
        |J(\theta_t) - J(\theta_-)|\, \1{\theta_t \in \Theta_{t-1}}
        &\leq |X_t\tran(\theta_t-\theta_-)|\,\1{\theta_t \in \Theta_{t-1}} \tag{\cref{claim:subgrads}} \\
        &\leq \norm[\hH_{t-1}^{-1}]{X_t} \norm[\hH_{t-1}]{\theta_t - \theta_-} \1{\theta_t \in \Theta_{t-1}} \tag{Cauchy-Schwarz}\\
        &\leq \frac{21}{10} C_d(\delta')\gnd \norm[\hH_{t-1}^{-1}]{X_t}\,. \tag{definition of $\Theta_{t-1}$ and \cref{lem:event-E2}}
    \end{align*}
    When combined with the inequalities developed thus far, this yields the statement of the claim.
\end{proof}

\subsection{Bounding $\rpred$ (proof of \cref{lem:bound-rpred})}\label{apx:rpred}

\begin{proof}[Proof of \cref{lem:bound-rpred}]
    Observe that
    \begin{equation*}
        \rpred
        =\mu(X_t^\top\theta_t) - \mu(X_t^\top \theta_\star)
        \leq |\mu(X_t^\top \theta_t) - \mu(X_t^\top \hat\theta_{t-1})| + |\mu(X_t^\top\hat\theta_{t-1}) - \mu(X_t^\top \theta_\star)|.
    \end{equation*}
    Now, by Taylor's theorem and self-concordance,
    \begin{align*}
        \vert \mu(X_t^\top \hat\theta_{t-1}) - \mu(X_t^\top \theta_\star)\vert
        &\leq (\dot\mu(X_t\tran \hat\theta_{t-1}) + |\ddot\mu(\xi)|D(\hat\theta_{t-1}- \theta_\star))|X_t\tran(\hat\theta_{t-1} - \theta_\star)|\\
        &\leq (\dot\mu(X_t\tran \hat\theta_{t-1}) + \dot\mu(\xi) M D(\hat\theta_{t-1}- \theta_\star)) |X_t\tran(\hat\theta_{t-1} - \theta_\star)| \numberthis \label{eq:mudiffb1}
    \end{align*}
    for some $\xi$ between $X_t\tran\hat\theta_{t-1}$ and $X_t\tran \theta_\star$.
    Recall that on $\cEd \cap \{t>\tau\}$ by \cref{claim:b-D-star},
    \begin{align}
	M D(\hat\theta_{t-1}- \theta_\star) \leq 1 \,.\label{eq:mdbone}
	\end{align}
	Hence,
    \begin{align*}
       \MoveEqLeft \dot\mu(\xi)M D(\hat\theta_{t-1}- \theta_\star)
        \leq
        \dot\mu(\xi) \\
        &\le
        \exp(M |\xi - X_t^\top\hat\theta_{t-1} | )\dot\mu(X_t^\top \hat\theta_{t-1})
        	\tag{\cref{claim:dot-mu-exp}}\\
        & \leq \exp(MD(\hat\theta_{t-1}- \theta_\star) )\dot\mu(X_t^\top \hat\theta_{t-1})
        \tag{$\xi$ is between $X_t\tran\theta_\star$ and $X_t\tran\hat\theta_{t-1}$, definition of $D(\cdot)$}\\
        & \leq e\dot\mu(X_t^\top \hat\theta_{t-1}) \,. \tag{by \cref{eq:mdbone}}
    \end{align*}
    which, together with \cref{eq:mudiffb1} gives
    \begin{align*}
	| \mu(X_t^\top \hat\theta_{t-1}) - \mu(X_t^\top \theta_\star)|
    &\le (e+1)\dot\mu(X_t\tran \hat\theta_{t-1})
	|X_t\tran(\hat\theta_{t-1} - \theta_\star)| \\
    &\leq (e+1)\dot\mu(X_t\tran \hat\theta_{t-1}) \norm[\hH_{t-1}]{\hat\theta_{t-1} - \theta_\star} \norm[\hH_{t-1}]{X_t} \tag{Cauchy-Schwarz}\\
    &\leq \sqrt{\frac{11}{10}}(e+1)\dot\mu(X_t\tran \hat\theta_{t-1}) \norm[H_{t-1}]{\hat\theta_{t-1} - \theta_\star} \norm[\hH_{t-1}]{X_t} \tag{by \cref{lem:event-E1}, $\hH_{t-1}\preceq \sqrt{\frac{11}{10}} H_{t-1}$} \\
    &\leq 5\gnd \dot\mu(X_t\tran\hat\theta_{t-1})\norm[\hH^{-1}_{t-1}]{X_t} \tag{other part of \cref{lem:event-E1}}\,,
	\end{align*}
    where in the last inequality we also bounded $\sqrt{\frac{11}{10}} \times \frac{21}{20} \times (e+1) \leq 5$.

    By the same sequence of arguments, this time using \cref{lem:event-E2} in place of \cref{lem:event-E1},
    \begin{equation*}
        | \mu(X_t^\top \theta_t) - \mu(X_t^\top \hat\theta_{t-1})| \leq 5C_d(\delta')\gnd \dot\mu(X_t\tran\hat\theta_{t-1}) \norm[\hH^{-1}_{t-1}]{X_t}.
    \end{equation*}
    Summing the two bounds per the initial decomposition of $\rpred$ of this proof gives the stated bound.
\end{proof}

\subsection{Bounding sum of derivatives (proof of \cref{lem:dot-mu-sum})}\label{apx:dot-mu-sum}

\begin{proof}
    By a first order Taylor expansion of $\dot\mu$,
    \begin{align*}
        \sum_{t=\tau+1}^n \dot\mu(X_t^\top \theta_\star)
        &= \sum_{t=\tau+1}^n \dot\mu(x_\star^\top \theta_\star) + \sum_{t=\tau+1}^n (X_t-x_\star)^\top \theta_\star \int_{v=0}^1 \ddot\mu\left(x_\star^\top \theta_\star + v(X_t-x_\star)^\top \theta_\star\right)dv  \\
        &\leq n\dot\mu(x_\star^\top \theta_\star) + \sum_{t=\tau+1}^n \left|  (X_t-x_\star)^\top \theta_\star \int_{v=0}^1 \ddot\mu\left(x_\star^\top \theta_\star + v(X_t-x_\star)^\top \theta_\star\right)dv \right| \\
        &\leq  n\dot\mu(x_\star^\top \theta_\star) + \sum_{t=\tau+1}^n (x_\star-X_t)^\top \theta_\star  \int_{v=0}^1 \left| \ddot\mu\left(x_\star^\top \theta_\star + v(X_t-x_\star)^\top \theta_\star\right) \right| dv
        \tag{$X_t^\top \theta_\star \leq x_\star^\top \theta_\star $} \\
        &\leq  n\dot\mu(x_\star^\top \theta_\star) + M\sum_{t=\tau+1}^n (x_\star-X_t)^\top \theta_\star \int_{v=0}^1 \dot\mu\left(x_\star^\top \theta_\star + v(X_t-x_\star)^\top \theta_\star\right)dv
        \tag{$|\ddot\mu|\leq M\dot\mu$} \\
        &=  n\dot\mu(x_\star^\top \theta_\star) + M\sum_{t=\tau+1}^n (\mu(x_\star^\top\theta_\star) - \mu(X_t^\top \theta_\star))
        \tag{fundemental theorem of calculus}\\
        &= n \dot\mu(x_\star^\top \theta_\star) + M \wR. \tag*{\qed}
    \end{align*}\renewcommand{\qedsymbol}{}
\end{proof}

\clearpage

\end{document}